\documentclass[12]{article}
\usepackage{algorithm}
\usepackage{algpseudocode}

\usepackage{authblk}

\usepackage{comment}
\usepackage{amsmath}
\usepackage{amssymb}
\usepackage{amsthm}
\usepackage{bm}
\usepackage{mathrsfs}
\usepackage{graphicx}
\usepackage{tikz}
\usetikzlibrary{arrows, arrows.meta}
\usepackage{nicefrac}
\usepackage{wrapfig}
\usepackage[numbers]{natbib}
\usepackage{mathtools}
\usepackage{amsbsy}

\usepackage{xcolor}
\definecolor{dkblue}{cmyk}{1,.54,.04,.19} 

\usepackage{hyperref} 

\usepackage{geometry}
\hypersetup{
  bookmarks=true,         
    unicode=false,          
    pdftoolbar=true,        
    pdfmenubar=true,        
    pdffitwindow=false,     
    pdfstartview={FitH},    
    pdftitle={Invariant Lipschitz Bandits: A Side Observation Approach},    
    pdfauthor={Lattimore},     
    pdfsubject={Bandits},   
    pdfcreator={pdflatex},   
    pdfproducer={Producer}, 
    pdfkeywords={bandits} {statistics} {machine learning}, 
    pdfnewwindow=true,      
    colorlinks=true,        
    linkcolor=black,       
    citecolor=dkblue,       
    filecolor=dkblue,       
    urlcolor=dkblue,        
}

\providecommand{\keywords}[1]
{
  \small	
  \textbf{\textit{Keywords---}} #1
}

\usepackage[textsize=tiny]{todonotes}

\usepackage[capitalise]{cleveref}
\usepackage[bf]{caption}
\theoremstyle{plain}
\newtheorem{theorem}{Theorem}

\newtheorem{lemma}[theorem]{Lemma}
\newtheorem{proposition}[theorem]{Proposition}

\theoremstyle{definition}
\newtheorem{definition}[theorem]{Definition}
\newtheorem{example}[theorem]{Example}
\newtheorem{remark}[theorem]{Remark}
\theoremstyle{remark}

\title{Invariant Lipschitz Bandits: A Side Observation Approach}
\date{}

\usepackage{amsmath}

\usepackage{dirtytalk}

\usepackage{comment}

\newtheorem{assumption}{Assumption}

\newcommand{\bb}[1]{\mathbb{#1}}

\newcommand{\Ball}{\mathcal{B}}
\newcommand{\PackingNumber}{N^{\mathrm{pack}}}
\newcommand{\CoveringNumber}{N^{\mathrm{cov}}}
\newcommand{\Distance}{\mathcal{D}}

\newcommand{\Vol}{\mathrm{Vol}}

\newcommand{\SymmetryGroup}[1]{\mathrm{Sym}\left(  #1 \right)}

\newcommand{\DirichletFD}{\mathbf{D}}
\newcommand{\FundDomain}{\mathbf{F}}

\newcommand{\Projection}{\mathrm{Proj}}
\newcommand{\Euclidean}{E}
\newcommand{\Expectation}{\bb{E}}
\newcommand{\Closure}[1]{\overline{#1}}

\newcommand{\Regret}{\mathbf{R}}
\newcommand{\RdRegret}{\hat{\Regret}}
\newcommand{\ArmSet}{\mathcal{X}}
\newcommand{\Graph}{G}
\newcommand{\Vertices}{V}
\newcommand{\Edges}{E}
\newcommand{\NeiOrbit}{\mathrm{N}}

\newcommand{\Arm}{X}
\newcommand{\Reward}{Y}
\newcommand{\ExpReward}{f}

\newcommand{\SymG}{\mathcal{G}}
\newcommand{\TotalRound}{n}

\newcommand{\UCB}{\mathrm{U}}
\newcommand{\PlayedTime}{T}
\newcommand{\ObservedTime}{O}
\newcommand{\Strategy}{\mathcal{A}}

\newcommand{\CliqueCoveringNumber}{\chi}
\newcommand{\IndependentNumber}{\alpha}

\newcommand{\Clique}{\mathcal{C}}
\newcommand{\IndSet}{\mathcal{I}}
\newcommand{\CoveringDirichletFD}{\Vertices_{\overline{\DirichletFD}}}
\newcommand{\CoveringDirichletFDInside}{\Vertices_{\overline{\DirichletFD},1}}
\newcommand{\CoveringDirichletFDBoundary}{\Vertices_{\overline{\DirichletFD},2}}

\newcommand{\ApprExpReward}{\hat{\ExpReward}}

\newcommand{\PackingPointsFD}{U}
\newcommand{\PackingPointsFarFromBoundaryFD}{W}

\newcommand{\ClassExpReward}{\mathcal{F}}
\newcommand{\ClusterOrbit}{\mathbf{C}}

\author[1]{Nam Phuong Tran}
\author[1]{Long Tran-Thanh}

\affil[1]{Department of Computer Science, University of Warwick}

\begin{document}
\title{Invariant Lipschitz Bandits: A Side Observation Approach}
%
%

%
%
%
\maketitle              
\begin{abstract}
Symmetry arises in many optimization and decision-making problems, and has attracted considerable attention from the optimization community: by utilizing the existence of such symmetries, the process of searching for optimal solutions can be improved significantly. 
Despite its success in offline settings, the utilization of symmetries has not been well examined within online optimization problems, especially in the bandit literature.
As such, in this paper, we study the invariant Lipschitz bandit setting, a subclass of the Lipschitz bandits in which a group acts on the set of arms and preserves the reward function.
We introduce an algorithm named \texttt{UniformMesh-N}, which naturally integrates side observations using group orbits into the 
uniform discretization algorithm \cite{Kleinberg2005_UniformMesh}.
Using the side-observation approach, we prove an improved regret upper bound, which depends on the cardinality of the group, given that the group is finite.
We also prove a matching regret's lower bound for the invariant Lipschitz bandit class (up to logarithmic factors).
We hope that our work will ignite further investigation of symmetry in bandit theory and sequential decision-making theory in general.

\keywords{Bandit Theory  \and Symmetry \and Group Theory.}
\end{abstract}
%
%
%
\section{Introduction}


%

Stochastic Multi-Armed Bandit (MAB) is a typical model of a sequential decision-making problem under uncertainty.
The main challenge of these decision-making problems is uncertainty. That is, the outcome of a decision is only revealed after the decision is made with the possible presence of noise.
Therefore, to achieve a high total reward, the agent (or the decision maker) must collect information along with trials and errors (i.e., \textit{exploring}), while using this information to choose a decision with a high expected reward (i.e., \textit{exploiting}). 
The exploration-exploitation trade-off is the main problem that one usually faces in such decision-making problems.
The performance of a bandit strategy is typically evaluated in terms of regret which is the difference between the total expected reward of the used strategy and that of the optimal strategy.
Stochastic MAB has been extensively studied since, and have achieved great success, both theoretically and practically \cite{Thompson1933_FirstPaper,Herbert1952_FirstPaper}. 
For textbook treatment of the subject, see, e.g.,  \cite{Lattimore2020_BanditBook,Slivkins2019_BanditBook}.

Classical stochastic MAB models assume a finite set of arms and unstructured environment, that is, information of an arm gives no information of others. 
This assumption is rather simple and might not capture the intrinsic relations between the arms.
An immediate consequence is that MAB problems with a large or even infinite set of arms are difficult or intractable to address without any further assumption.
This leads to a significant body of literature studying structured MAB, including linear bandits \cite{Yasin2011_LinearBandit}, convex optimizations with bandit feedback \cite{Shamir2013_ConvexBandit}, and Lipschitz bandits \cite{Kleinberg2019_MetricBandit_Zooming,Bubeck2011_X_armed}.
With these additional assumptions, algorithms can be designed to exploit the underlying structures to achieve significantly smaller regret, compared to that of algorithms that are oblivious to these structures.

In addition to the abovementioned structures on the set of arms, many decision-making and optimization problems exhibit symmetry, that is, the expected reward function and the set of arms are invariant under certain group of transformations.
For example, an important real-world application of continuous bandits, which naturally inherit symmetry, can be found in the context of online matrix factorization bandits \cite{Li2019_SymmtryMatrixFactorisation}. The online matrix factorization bandit problem, in turn, has important machine learning applications, such as in interactive recommendation systems \cite{Kawale2015_MatrixFactoriastionBandit,Wang2017_MF_RS,Wang2019_OnlineMatrixFactorisation}, or the online dictionary learning problem \cite{Lyu20_DICLearning,Mairal2010_DICLearning}.
Briefly speaking, in recommender systems the matrix factorization method aims to decompose the user-item rating matrix $R$ into a product of small matrices, that is, $R = HW^\top$, where $H$ is the user-feature matrix and $W$ is the item-feature matrix. 
The goal is to find a pair $(H,W)$ such that the matrix $R$ can be recovered. 
Now, for any orthogonal matrix $\phi$ of the feature space, it can be easily seen that if any pair $(H,W)$ can recover matrix $R$, the pair $(H\phi,W\phi)$ can also recover matrix $R$. 
Therefore, this problem is invariant with respect to the action of the group of orthogonal matrices.
In fact, as observed in \cite{Li2019_SymmtryMatrixFactorisation}, the appearance of symmetry in matrix factorization gives rise to infinitely many saddle points in the optimization landscape, causing a great challenge for an optimization algorithm that is oblivious to the existence of symmetry.
We discuss matrix-factorization bandit and the implications of our result on this problem in Appendix \ref{Appendix: Real-world application} as it might be of independent interest. 

Another motivating example of symmetry in bandit settings is the online job-machine allocation problem with identical machines.
In this problem,  if the allocation is swapped among these identical machines, the total performance remains unchanged, and hence the problem is permutation-invariant.
Intuitively, algorithms that are oblivious to this symmetry may waste time exploring many symmetric solutions, resulting in slower convergence.

While symmetry has not been investigated in the bandit literature, it has been studied extensively in the optimization domain, especially in Integer Linear Programming (ILP) \cite{Margot2010_SymmetricILP_Survey}.
In these problems, where the majority of algorithms are based on the branch-and-bound principle, the presence of symmetry causes an extremely waste of time in branching symmetric solutions, if it is not handled explicitly.
The reason for this is that as symmetric solutions yield the same objective function, they cannot be pruned by the bounding principle.
Therefore, the presence of symmetry causes great difficulty for existing combinatorial optimization algorithms that are oblivious to symmetry. 
As such, various approaches have been proposed to deal with existing symmetry in the model, including adding symmetry-breaking constraints or pruning symmetric nodes during the branch-and-bound process.

Despite its ubiquity in optimization problems, to date symmetry has not been extensively studied in bandit problems, which can be considered as the online counterpart of many optimization settings. 
Note that for unstructured $K$-armed bandit problems, if we assume a group $\SymG$ acts freely on the set of arms and the reward function is invariant under the action of $\SymG$, then we can compute the fundamental domain (i.e., a subset of arms whose orbit covers the original set) and restrict the bandit algorithm on the subset. 
As there are exactly $K/|\SymG|$ disjoint orbits, the regret of standard algorithms such as $\texttt{UCB1}$ \cite{Auer2002_UCB1} scales as $\Tilde{\mathcal{O}} \left(\sqrt{\frac{K\TotalRound}{|\SymG|}}\right)$ where $\TotalRound$ is the total number of rounds.
Therefore, one could expect that the same improvement can be easily achieved in the case of other bandit settings. 
However, as we shall point out in detail in Subsection \ref{subsec: invariant MAB and difficulty of Lipschitz bandit}, the same reasoning applied to the invariant $K$-armed bandits is not feasible in the case of Lipschitz bandits.
This is due to two main reasons: 
(i) As the set of arms we consider is a dense subset of $\bb{R}^d$, simply counting disjoint orbits to show the reduction in cardinality of the fundamental domain is not feasible;
(ii) The subgroup of Euclidean isometries we consider in this paper may admit uncountably infinite fixed points.
Therefore, it is not trivial that the combination of symmetry with other structures such as Lipschitz continuity might help algorithms achieve better regret.

Now, in theory, if a Lipschitz bandit problem has the expected reward function to be invariant under the action of a finite group $\SymG$ of Euclidean isometries, one naive approach to exploit symmetry is to construct the fundamental domain and sample within the domain. 
However, sampling in the fundamental domain is computationally impractical, as we shall point out in detail in \ref{subsec: invariant MAB and difficulty of Lipschitz bandit}.
Briefly, constructing a fundamental domain requires placing global inequalities on the original set of arms, and the problem of finding value for an arm that satisfies all global constraints is known as a NP-complete problem \cite{Bessiere2007_GlobalConstraints}. 

Against this background, we ask the question whether we can design an efficient bandit algorithm to tackle the Invariant Lipschitz Bandit (ILB) problem without the need of constructing the fundamental domain (to avoid facing the computationally expensive task of sampling from that domain). 
In particular, we ask that whether utilizing invariance can result in tighter regret bounds, if 
the only assumption of the model is that the set of arms is equipped with a distance, and the expected reward is Lipschitz continuous w.r.t. the set of arms. 

\subsection{Our Contributions}

The contribution of this paper can be summarized as follows.
\begin{enumerate}
    \item We introduce a new model for ILB problems, where the reward and the set of arms are invariant under a known finite group of transformations.
    \item We introduce an algorithm, named \texttt{UniformMesh-N}, which naturally integrates side observations using group orbits into the uniform discretization algorithm.
    \item We provide a regret analysis for the algorithm which shows an improvement in regret by a factor depending on the cardinality of the group. In particular, we show that if the number of rounds $\TotalRound$ is large enough, the regret of \texttt{UniformMesh-N} algorithm is at most $\Tilde{\mathcal{O}}\left( \left(\frac{1}{|\SymG|}\right)^{\frac{1}{d+2}} \TotalRound^{\frac{d+1}{d+2}} \right)$, where $\SymG$ is the transformation group and $\Tilde{\mathcal{O}}$ only hides logarithmic factor. 
    \item We provide a new lower bound for this bandit class, that is $\Omega\left( \left(\frac{1}{|\SymG|}\right)^{\frac{1}{d+2}} \TotalRound^{\frac{d+1}{d+2}} \right)$, given $\TotalRound$ is large enough. This lower bound essentially matches the upper bound of the algorithm up to a logarithimic factor.
\end{enumerate}


Among these, our most important contribution is the regret analysis, which we proved a strict improvement of leveraging symmetry in designing algorithm. 
Note that, in supervised learning where symmetry is widely studied, while many empirical work showed that empirical work shows that integrating symmetry into learning can help reduce the generalization error, from a theoretical perspective, proving strict improvement when using symmetry is known as a challenging problem within the machine learning community \cite{Elesedy2021_SymmetryInSL,behboodi2022a_SymmetryInSL,Sannai2019_PermuatationCube}. 
To our knowledge, our analysis is novel and this paper is the first to show a strict improvement of using symmetry in the sequential decision-making context.

\subsection{Related Works}

Lipschitz MAB problem is referred to as continuum-armed bandit problems in early works \cite{Agrawal1995_ContinuumArmed,Kleinberg2005_UniformMesh,Auer2007_ImprovedUniformMesh}, where the set of arms is the interval $[0,1]$. 
Within this context, \cite{Kleinberg2005_UniformMesh} proposed a simple uniform discretization-based algorithm, which is later referred to as \texttt{UniformMesh} algorithm. 
Despite its simplicity, \cite{Kleinberg2005_UniformMesh} showed that with the global Lipschitz continuous reward function, the \texttt{UniformMesh} algorithm achieves $\Tilde{\mathcal{O}}(\TotalRound^{2/3})$ regret, which essentially matches the minimax lower bound up to a logarithmic factor. 
However, uniform discretization is potentially wasteful as it keeps the approximation uniformly good over the whole space, while it is probably better to adapt the approximation so that the approximation is more precise around the maxima while being more loose in the rest of the space.
With this in mind, \cite{Kleinberg2019_MetricBandit_Zooming} and \cite{Bubeck2011_X_armed} proposed adaptive discretization algorithms, which are referred to \texttt{Zooming} algorithm and \texttt{HOO} algorithm, respectively.
These adaptive discretization algorithms achieve better regret when dealing with benign instances, while not being deteriorated in the worst case. 
In particular, \cite{Kleinberg2019_MetricBandit_Zooming,Bubeck2011_X_armed} showed that the regret of these algorithms is at most $\Tilde{\mathcal{O}}(\TotalRound^{\frac{d'+1}{d'+2}})$, where $d'$ is \textit{near-optimal dimension} and is upper-bounded by \textit{covering dimension}. 
Our algorithm is based on \texttt{UniformMesh} algorithm together with side observation. 
The rationale behind this choice is that \texttt{UniformMesh} fixes the set of played points and hence can easily combine with the group orbit to create a fixed undirected graph, which cannot be done if the set of played points (i.e., the set of vertices) varies every round.
The combination of group orbit with adaptively discretizing algorithms such as \texttt{HOO} or \texttt{Zooming} algorithms may require a more complicated treatment and is left for future work.

Side-observation feedback was first introduced in the adversarial MAB context \cite{Mannor2011_Adv_WithGraph_First}, and has been further studied by various authors \cite{Alon2013_Adv_graph,Alon2017_Adv_UnIn,Cohen2016_Sto_AE_WithoutGraph}. 
Especially, the minimax lower bound as in \cite{Mannor2011_Adv_WithGraph_First} is $\Omega\left(\sqrt{\IndependentNumber \TotalRound}\right)$, where $\IndependentNumber$ is the independent number of the graph.
The first work that examines side observations in a stochastic MAB setting is \cite{Caron2012_Sto_UCBN_WithGraph}, in which the observation graph is undirected and fixed over time. 
The authors introduced a naive algorithm named \texttt{UCB-N}, which chooses an arm $x$ that maximizes the index $\mathrm{UCB}$, then observes all the rewards in the neighborhood of $x$. 
\cite{Caron2012_Sto_UCBN_WithGraph} proved the gap-dependent regret bound depends is at most $\mathcal{O}\left( \log(\TotalRound) \sum_{C \in \Clique} \frac{\max_{x\in C} \Delta_x}{\min_{x\in C}\Delta_x^2}  \right)$, where $\mathcal{C}$ is a clique covering of the graph and $\Delta_x$ is the subopitimality gap of arm $x$.
More recently, \cite{Lykouris2019_Sto_UCBN_WithoutGraph} provided an alternative analysis to achieve a tighter regret for \texttt{UCB-N}, $\Tilde{\mathcal{O}}\left(\log{\TotalRound} \sum_{x\in \IndSet} \frac{1}{\Delta_x} \right)$, where $\IndSet$ is an independent set of $\Graph$. 
From this, they also derived a minimax upper bound as $\Tilde{\mathcal{O}} \left(\sqrt{\IndependentNumber \TotalRound} \right)$.
Similarly, \cite{Cohen2016_Sto_AE_WithoutGraph} an algorithm based on arm elimination that achieves gap-dependent regret as $\mathcal{O}\left( \log{\TotalRound} \sum_{x\in \IndSet} \frac{1}{\Delta_x} \right)$, and minimax regret as $\Tilde{\mathcal{O}}\left( \sqrt{\IndependentNumber\TotalRound} \right)$.
Our algorithm \texttt{UniformMesh-N} is based on \texttt{UCB-N} algorithm, and our analysis technique is similar to \cite{Caron2012_Sto_UCBN_WithGraph}.
The reason for this choice is that the graph induced by group action is highly "regular", namely, the covering number of this graph is dependent on the covering number of the fundamental domain.
Therefore, the result of \cite{Caron2012_Sto_UCBN_WithGraph} is sufficient to introduce the cardinality of the group into the regret.

Symmetry has been studied for decades in the literature of (Mixed) Integer Linear Programming in the case of subgroups of the group of permutation matrices \cite{Margot2010_SymmetricILP_Survey}.
There are two main approaches to dealing with symmetry; the main idea is to eliminate symmetric solutions to reduce search space. 
The first approach is to add symmetry-breaking constraints to the domain, making symmetric solutions infeasible \cite{Walsh2012_SymmetryBreakingSurvey}. 
The second approach is to incorporate symmetry information during the tree search procedure (e.g., branch and bound) to prune symmetric branches, for example, isomorphism pruning \cite{Margot2002_IsomorphismPruning}, orbital branching \cite{Ostrowski2011_OrbitalBranching}.
The concern of these works is to ensure that at least one of the optimal solutions remains feasible.
However, these works did not compare the convergence rate of these algorithms, which uses symmetry information, with algorithms that are oblivious to symmetry information.
In contrast, we analyze the convergence rate of the algorithm in terms of regret and are able to show an improvement of the algorithm which incorporates the group information. 
Although the simple mechanism of uniform discretization allows us to provide the regret bound, integrating group information with tree-based algorithms (e.g., isomorphism pruning \cite{Margot2002_IsomorphismPruning} or orbital branching \cite{Ostrowski2011_OrbitalBranching}) is an interesting research direction and is left for future work.

\subsection{Outline}
The rest of this paper is organized as follows.
Notations and some important proporties of subgroup Euclidean isometries are provided in Section 2, followed by the mathematical formulation of the ILB problem.
In Section 3 we formalize the construction of the graph induced by the group action; then the \texttt{UniformMesh-N} algorithm is introduced, followed by its regret analysis. 
In Section 4, a regret lower bound for the ILB problem is provided.
The limit of this paper and further work are discussed in Section \ref{sec: discussion}.
Due to the space limit, we defer the most of the proofs to the appendix, as well as the more detailed description of our algorithm (Appendix~\ref{Appendix: Suggestion of Design and Implementation of Algorithm}), and the application of our results to the matrix factorization bandit problem (Appendix~\ref{Appendix: Real-world application}).

\section{Preliminary and Problem Setup}
\subsection{Preliminary}
Let $\ArmSet$ be a compact and convex subset of a $d$-dimensional Euclidean space $\Euclidean^d$; denote $\Distance$ as Euclidean distance. For any $x\in \ArmSet$ and $\delta >0$, define a $\delta$-open ball as $\Ball(x,\delta) = \{x' \in \ArmSet \mid \Distance(x,x') < \delta \}$. 

\paragraph{Packing and Covering properties.}
Consider $\delta>0$, $S \subseteq \ArmSet$, and a finite subset $K \subset \ArmSet$. 
$K$ is called a \textit{$\delta$-packing} of $S$ if and only if $K \subset S$ and $\Distance(i,j) > \delta$ for any two distinct points $i,j\in K$.
$K$ is called a \textit{$\delta$-covering} of $S$ if and only if for any $x \in S$, there is $i \in K$ such that  $\Distance(i,x) < \delta$.
$K$ is called \textit{$\delta$-net} of $S$ if and only if $K$ is both a $\delta$-covering and $\delta$-packing of $S$.
$K$ is called a \textit{strictly $\delta$-packing} of $S$ if and only if $K$ is an $\delta$-packing of $S$ and $\Distance(i,\partial S) > \delta, \forall i\in K$, where $\partial S$ is the topological boundary of $S$ in $\Euclidean^d$.  
The \textit{$\delta$-packing number} $\PackingNumber(S, \Distance, \delta)$ of $S \subseteq \ArmSet$ with respect to the metric $\Distance$ is the largest integer $k$ such that there exists a $\delta$-packing of $S$ whose cardinality is $k$.
The \textit{$\delta$-covering number} $\CoveringNumber(S, \Distance, \delta)$ of $S \subseteq \ArmSet$ with respect to the metric $\Distance$ is the smallest integer $k$ such that there exists a $\delta$-covering of $S$ whose cardinality is $k$.
By the definitions of packing and covering number, if $K$ is a $\delta$-net of $S$, then $\CoveringNumber(S,\Distance,\delta) \leq |K| \leq \PackingNumber(S,\Distance,\delta)$. 
Moreover, we can bound the packing and covering numbers of a compact set as follows.

\begin{proposition}[\cite{Yihong_LecturePackingCoveringNumber}] \label{prop: LB and UB for covering and packing number}
Let $\Ball$ be the unit ball and $S$ be a compact subset of $\Euclidean^d$. Then, for any $\delta>0$, we have
\begin{equation*}
\begin{aligned}
    \left(\frac{1}{\delta}\right)^{d} \frac{\Vol(S)}{\Vol(\Ball)} &\leq \CoveringNumber(S,\Distance,\delta)
    \leq \PackingNumber(S,\Distance,\delta) \leq \left(\frac{3}{\delta}\right)^{d} \frac{\Vol(S)}{\Vol(\Ball)}.
\end{aligned}
\end{equation*}
\end{proposition}

\paragraph{Euclidean isometries.}
A bijection mapping $\phi: \Euclidean^d \rightarrow \Euclidean^d$ is a \textit{Euclidean isometry} if and only if it preserves the Euclidean distance, that is, $\Distance\left(\phi(x), \phi(x')\right) = \Distance(x, x')$.
The symmetry group of $\ArmSet$, denoted as $\SymmetryGroup{\ArmSet}$, is the set of all Euclidean isometries that preserve $\ArmSet$, that is, 
\begin{equation*}
    \begin{aligned}
         \SymmetryGroup{\ArmSet} &= \{ \phi: \Euclidean^d \rightarrow \Euclidean^d 
         \mid \text{$\phi$ is a Euclidean isometry and $\phi(\ArmSet) = \ArmSet$} \}.
    \end{aligned} 
\end{equation*}

For any group element $\phi \in \SymmetryGroup{\ArmSet}$, denote $\phi(x)$ as $\phi \cdot x$. 
Let $\SymG$ be a finite subgroup of $\SymmetryGroup{\ArmSet}$, we write $\SymG \leq \SymmetryGroup{\ArmSet}$, where $\leq$ denotes the subgroup relation.
Denote the stabilizer of $x $ in $\SymG$ as $\SymG_x = \{g \in \SymG \mid g\cdot x = x\}$. 
Denote the $\SymG$-orbit of $x$ as $\SymG \cdot x = \{g \cdot x \mid g \in \SymG \}$ and the $\SymG$-orbit of a subset $S$ as $\SymG \cdot S = \{g \cdot x \mid g \in \SymG, x\in S \}$.
As stated in \cite{Ratcliffe07_GeometryBook}, any finite group of isometries of Euclidean space must admit a point $x \in \Euclidean^d$ whose stabilizer $\SymG_x$ is trivial.

For a subset $S \subset \Euclidean^d$, denote the closure of $S$ in $\Euclidean^d$ as $\Closure{S}$. 
Note that since $\ArmSet$ is compact and then closed in $\Euclidean^d$, the closure of $S$ in $\ArmSet$ is equivalent to the closure of $S$ in $\Euclidean^d$.
Next, we state the definition of the fundamental domain, an important notion that we use throughout the paper, especially in the regret analysis. 

\begin{definition}[Fundamental domain]
A subset $\FundDomain \subset \Euclidean^d$ is a fundamental domain for a finite group $\SymG$ of isometries of a Euclidean space $\Euclidean^d$ if and only if 
(1) the set $\FundDomain$ is open in $\Euclidean^d$;
(2) the members of the set $\{g\cdot \FundDomain \mid g \in \SymG\}$ are mutually disjoint;
(3) $\Euclidean^d = \bigcup_{g\in \SymG} g\cdot  \Closure{\FundDomain}$;
(4) $\FundDomain$ is connected.
\end{definition}

Consider a finite subgroup $\SymG$ of isometries of $\Euclidean^d$, we can explicitly construct a fundamental domain for $\SymG$.
Recall that $\SymG$ admits a point $x^o \in \Euclidean^d$ whose stabilizer $\SymG_{x^o}$ is trivial. 
As in \cite{Ratcliffe07_GeometryBook}, a fundamental domain can be constructed as follows.

\begin{definition}[Dirichlet domain]\label{def: Dirichlet domain}
Let $x^o$ be a point in $\Euclidean^d$ such that $\SymG_{x^o}$ is trivial. 
For each $g\in \SymG$, define an open half-space as

\begin{equation} \label{eq: Half space}
    H_g (x^o) = \{x \in \Euclidean^d \mid \Distance(x,x^o) < \Distance(x, g\cdot x^o) \}.
\end{equation}

A Dirichlet domain is the intersection of those open half-spaces, particularly
\begin{equation}
    \FundDomain = \bigcap_{g\in \SymG} H_g (x^o).
\end{equation}
\end{definition}
As shown in \cite{Ratcliffe07_GeometryBook}, a Dirichlet domain is indeed a fundamental domain.
Note that since the boundary of $\FundDomain$ is a subset of a finite union of some $(d-1)$-dimensional hyperplanes, it has volume zero.

Given Dirichlet domain $\FundDomain$ of $\SymG$ in $\Euclidean^d$. 
Define $\DirichletFD = \FundDomain \cap \ArmSet$. 
Now, we show that $\DirichletFD$ is also a fundamental domain for $\SymG$ acting in $\ArmSet$.

\begin{proposition}\label{prop: Proper FD of ArmSet}
$\DirichletFD$ is a fundamental domain for $\SymG$ acting in $\ArmSet$. In particular,
(i) $\DirichletFD$ is open in $\ArmSet$;
(ii) the members of $\{g\cdot \DirichletFD\}_{g\in \SymG}$ are mutually disjoint;
(iii) $\DirichletFD$ is connected;
(iv) $\ArmSet = \bigcup_{g\in \SymG} g\cdot \Closure{\DirichletFD}$.
\end{proposition}
The Proof of Proposition \ref{prop: Proper FD of ArmSet} is deferred to Appendix \ref{Appendix: Proof of prop FD of ArmSet}.

\subsection{Problem formulation} \label{subsec: Problem formulation}
For any $k \in \bb{N}^+\setminus\{1\}$, denote $[k]$ as $\{1,...,k\}$.
Denote the number of rounds as $\TotalRound$, which is assumed to be known in advance. 
Each round $t \in [\TotalRound]$, the agent chooses an arm $\Arm_t \in \ArmSet$, then the nature returns a bounded stochastic reward that is sampled independently from an unknown distribution $\Reward_t \sim \bb{P}_{\Arm_t}$. 
Define the expected reward function as $\ExpReward(x) := \bb{E}\left[ \Reward_t \mid \Arm_t = x\right]$.
Assume that $\bb{P}_x$ has support in $[0,1]$ for all $x\in \ArmSet$.
A bandit strategy is a decision rule for choosing an arm $\Arm_t$ in round $t \in [n]$, given past observations up to round $t-1$. 
Formally, a bandit strategy is a mapping
$\Strategy: (\ArmSet \times [0,1])^\TotalRound \rightarrow \mathcal{P}(\ArmSet)$, where $\mathcal{P}(\ArmSet)$ is the set of probability measures over the set of arms $\ArmSet$.
A bandit strategy is deterministic if $\mathcal{P}(\ArmSet)$ is the set of Dirac probability measures.

Let $x^*$ be an optimal arm, associated with the optimal mean reward $\ExpReward^* = \ExpReward(x^*)$. 
In each round $t \in [\TotalRound]$, the agent chooses an arm $\Arm_t$ and incurs an immediate regret $\ExpReward^* - \ExpReward(\Arm_t)$. 
The agent's goal is to minimize the regret over $\TotalRound$, defined as

\begin{equation}
    \Regret_\TotalRound = \Expectation \left[ \sum_{t=1}^\TotalRound \ExpReward^* - \ExpReward(\Arm_t)  \right].
\end{equation}

Here, the expectation is taken over all random sources, including both the randomness in the rewards and the bandit strategy used by the agent. 
Next, we state the assumptions on the expected reward function that defines the ILB problems.

\begin{assumption}[Lipschitz continuity] \label{assp: Lipschitz}
For all $x, x'\in \ArmSet$, $\Distance(x,x') \geq |\ExpReward(x) - \ExpReward(x')|$.
\end{assumption}

\begin{assumption}[Invariant expected reward function] \label{assp: Invariance}
Given a finite subgroup $\SymG \leq \SymmetryGroup{\ArmSet}$, the expected reward function $\ExpReward$ is invariant under the group action of $\SymG$, that is, for all $g \in \SymG$, $\ExpReward(x) = \ExpReward(g \cdot x)$. 
Assume that $\SymG$ is revealed to the agent in advance.
\end{assumption}

Note that while Lipschitz continuity \ref{assp: Lipschitz} is the standard assumption that defines the class of Lipschitz bandits \cite{Kleinberg2019_MetricBandit_Zooming}, the invariance assumption \ref{assp: Invariance} is new in the bandit's literature. 
As the agent knows $\SymG$ in advance, he can benefit from the information to learn the near-optimal arms more quickly.

\subsection{A warm-up: invariant finite unstructured MAB} \label{subsec: invariant MAB and difficulty of Lipschitz bandit}
    
\paragraph{Invariant unstructured $K$-armed bandit.}
As a warm-up, consider the case of an invariant $K$-armed bandit with unstructured environment. 
The problem is almost the same as the problem described in Subsection \ref{subsec: Problem formulation}, however, there are only $K$ arms, and Lipschitz continuity (assumption \ref{assp: Lipschitz}) is dropped.
Now, assume that the group $\SymG$ acts freely on $[K]$, denote the quotient space as $[K]/\SymG$.
As each orbit contains exactly $|\SymG|$ arms and these orbits are disjoint, it is obvious that the number of disjoint orbits $\left|[K]/\SymG \right|$ is $K/|\SymG|$.

For each element $q$ in the quotient space $[K]/\SymG$, denote the orbit $[K]_q \subset [K]$ as the orbit corresponding to $q$, and denote $k_q$ as the orbit representation of $[K]_q$.
We have that the fundamental domain $\bigcup_{q\in [K]/\SymG} k_q$ must contain the optimal arm, and its cardinality is $K/\SymG$.
Now, if we apply the standard MAB algorithm such as \texttt{UCB1} \cite{Auer2002_UCB1} to the fundamental domain and ignore other arms, we can obtain the regret
 $\Tilde{\mathcal{O}}\left( \sqrt{\frac{K\TotalRound}{|\SymG|} } \right)$.

\paragraph{Difficulties of Invariant Lipschitz Bandit.} The simplicity of the above result for invariant unstructured MAB is based on partitioning the set of arms into disjoint orbits, which in turn is based on two important assumptions that are not feasible in the case of the invariant and continuous Lipschitz bandit.
First, as the set of arms $\ArmSet$ is a uncountable set, simply counting the cardinality of the quotient set $\ArmSet/\SymG$ to show the number of disjoint orbits as in the case of $K$-armed bandit is not possible. 
Second, as groups of Euclidean isometries might admit (uncountably infinite) fixed points (points which are fixed by all group elements) and non-effective point (points which are fixed by some group elements that are not indentity). 
Therefore, applying group action on some points in $\ArmSet$ may result in the same points, or obtaining orbits whose points stay arbitrarily close to each other.
In this case, the orbit of these nearly fixed points in the mesh might not produce useful information (see, e.g., Example \ref{example: Permutation group} in the appendix).

A naive way to exploit symmetry in the ILB setting is to construct the Dirichlet domain and sampling within this subset.
To show the reduction in terms of regret, proving that the covering number of the fundamental domain is $\mathcal{O}\left(\frac{\delta^{-d}}{|\SymG|}\right)$ as a direct consequence of our Lemma \ref{lem: upper bound for packing number} is sufficient.
However, as we mentioned in the introduction, sampling within the fundamental domain is computationally impractical. 
In particular, construction of the Dirichlet fundamental domain in the Lipschitz bandit case is equivalent to placing global inequality constraints on the set of arms as the construction in Definition \ref{def: Dirichlet domain}.
Given that verifying that a value of $x$ satisfies all $|\SymG|$ linear inequality constraints takes $|\SymG|d^2$ operation, \cite{Bessiere2007_GlobalConstraints} showed that the problem of finding a point $x \in \ArmSet$ that satisfies the global constraints is NP-complete, implying that sampling in the fundamental domain is computationally intractable.
Against this background, we seek another approach that does not require sampling within the fundamental domain.
Note that while in the rest of the paper we are not interested in the method of sampling within the fundamental domain, the geometric properties of the fundamental domain still play a crucial role in our analysis.

\section{\texttt{UniformMesh-N} algorithm and regret analysis}

In this section, we first explain the construction of the graph induced by the group action of $\SymG$.
According to this, \texttt{UniformMesh-N} algorithm is introduced, followed by its regret analysis.

\subsection{Graph structure induced by the group action}
\label{subsection: Graph structure induced by the group action}

We now explain how the group action can naturally induce a graph structure on the set of discretization points of $\ArmSet$. 
Intuitively, in round $t$, the agent chooses an arm $\Arm_t$, and nature returns a reward $\Reward_t$, and the agent can then use the returned reward $\Reward_t$ to update the node along the orbit $(\SymG \cdot \Arm_t)$. 
This is similar to the side-observation setting, where the nature returns not only a reward associating the chosen arm, but also the rewards of its neighbors in graph feedback.
As a result, the orbit of $\Arm_t$ can be treated as a neighborhood of $\Arm_t$ in a graph formed by the group action. 
Moreover, it turns out that the graph's structure induced by the group action is highly regular, particularly its clique covering depends on the covering number of fundamental domain. 

Given a number $\delta>0$, let us denote $\Vertices$ as a $\delta$-net of $\ArmSet$.
Given the Dirichlet domain $\DirichletFD$ of $\SymG$ acting on $\ArmSet$, let $\CoveringDirichletFD$ denote a smallest subset of $\Vertices$ that covers $\Closure{\DirichletFD}$.
By Proposition \ref{prop: LB and UB for covering and packing number}, it follows that $|\Vertices|=\Theta(\delta^{-d})$.
We now compare the \say{size} of $\overline{\DirichletFD}$ with $\ArmSet$ in terms of covering number as the follows.

\begin{lemma} \label{lem: upper bound for packing number}
For some numbers $c_1, c_2 >0$, the cardinality of $\CoveringDirichletFD$ is bounded as
\begin{equation}
    |\CoveringDirichletFD| \leq \frac{c_1\delta^{-d}}{|\SymG|} + c_2 \delta^{-(d-1)}.
\end{equation}
Moreover, there are constants $c_3, c_4>0$, such that for any $\delta: 0< \delta < c_4/|\SymG|$, one has
\begin{equation}
    |\CoveringDirichletFD| \leq \frac{c_3 \delta^{-d}}{|\SymG|}.
\end{equation}
\end{lemma}

Lemma \ref{lem: upper bound for packing number} is in a similar vein as Lemma $1$ in \cite{Sannai2019_PermuatationCube}. 
However, while \cite{Sannai2019_PermuatationCube} only proved for the permutation group, our Lemma \ref{lem: upper bound for packing number} can be applied to any finite subgroup of $\SymmetryGroup{\ArmSet}$, and therefore the permutation group is treated as a special case. Thus, for our analysis is it essential to use Lemma \ref{lem: upper bound for packing number} instead of the result from \cite{Sannai2019_PermuatationCube} (as our setting is more generic).

The generalization of Lemma \ref{lem: upper bound for packing number} comes from a different technique that exploits the fact that the boundary of the Drichlet domain has zero volume.
The proof is deferred to Appendix \ref{Appendix: Graph Proof}, but we provide the underlying idea of the proof here. 
First, we partition the set $\CoveringDirichletFD$ into two disjoint subsets.
The first subset is points in $\CoveringDirichletFD$ lying "strictly" inside the fundamental domain (whose distance to the boundary greater than $\delta$), and the second subset is all points in $\CoveringDirichletFD$ lying near the boundary. 
Now, it is clear that the cardinality of the former is lesser than $|\Vertices|/|\SymG|$, since there are $|\SymG|$ disjoint images of $\DirichletFD$ and we can always choose the image that contains the smallest number of points. 
The cardinality of the latter is derived from the fact that the volume of the area near the boundary shrinks as $\delta$ gets smaller; consequently, the packing number of this area is only bounded by $\mathcal{O}(\delta^{-(d-1)})$.

Lemma \ref{lem: upper bound for packing number} implies that the covering number of the Dirichlet domain is proportional to $1/|\SymG|$and can become significantly small when $|\SymG|$ is large. This is the key ingredient to the improvement in regret's bound.
We give examples of some symmetry groups to hightlight the importance and difficulty of Lemma \ref{lem: upper bound for packing number} in appendix \ref{Appendix: Example of Symmetry Group}. 



\noindent
Before stating the definition of the graph, we first make an observation as follows.

\begin{proposition}\label{Prop: Undirected edge}
Let $\SymG$ be a finite subgroup of $\SymmetryGroup{\ArmSet}$. For any $x,x' \in \ArmSet$ and a constant $\delta >0$, if there is an action $g \in \SymG$ such that $\Distance(g\cdot x, x') <\delta$, then $\Distance(x, g^{-1}\cdot x') <\delta$.
\end{proposition}

\begin{definition}
Given a constant $\varepsilon>0$, define the graph $\Graph_{\varepsilon} = (\Vertices, \Edges_{\varepsilon})$, where \\
$\Edges_\varepsilon = \left\{ (x,x') \mid x,x' \in \Vertices \text{ and } \exists g\in \SymG \text{ s.t. } \Distance(g\cdot x,x' ) < \varepsilon \right\}$.
\end{definition}

Note that as a result of Proposition \ref{Prop: Undirected edge}, it follows that if $(x,x') \in \Edges_{\varepsilon}$, then $(x',x) \in \Edges_{\varepsilon}$. 
Therefore, $G_\varepsilon$ is an undirected graph.

\begin{definition} \label{def: Neighborhood of Orbit}
Given a constant $\varepsilon>0$, a Neighborhood of Orbit $(\SymG \cdot x)$ in $\Vertices \subset \ArmSet$ with distance at most $\varepsilon$ is defined as
\[\NeiOrbit(i,\varepsilon) = \left\{x \in \Vertices \mid \exists g\in \SymG \text{ s.t. } \Distance(g\cdot i, x) \leq \varepsilon  \right\}. \]
\end{definition}

It is obvious that $\NeiOrbit(i,\varepsilon)$ is the neighborhood of $i$ in $\Graph_\varepsilon$. 
Besides, for a constant $\delta>0$, suppose $x\in \NeiOrbit(i,2\delta)$, let $g \in \SymG$ such that $\Distance(x,g\cdot i) < 2\delta$, we obtain
\begin{equation}\label{eq: Bound for difference of mean reward in neighborhood}
    \begin{aligned}
        |\ExpReward(x) - \ExpReward(i)| =  |\ExpReward(x) - \ExpReward(g \cdot i)|
        < \Distance(x, g\cdot i) < 2\delta.
    \end{aligned}
\end{equation}
Given a graph $\Graph_\varepsilon$, a clique in $\Graph_\varepsilon$ is a subset of vertices $C \subseteq \Vertices$ such that all arms in $C$ are neighbors.
Moreover, a clique covering of $\Graph_\varepsilon$ is a collection of cliques whose union is $\Vertices$.
The clique covering number of the graph $\Graph_\varepsilon$, denoted as $\CliqueCoveringNumber(\Graph_\varepsilon)$, is the smallest integer $k$ such that there exists a clique covering of $\Graph_\varepsilon$ whose cardinality is $k$.

\begin{proposition} \label{Prop: Neighborhood to Clique}
Given a constant $\delta >0$ and $i\in \Vertices$, the set $\NeiOrbit(i,\delta)$ forms a clique of graph $\Graph_{2\delta}$.
\end{proposition}

\begin{remark}
In order to prove the regret bound using \cite{Caron2012_Sto_UCBN_WithGraph}'s analysis, we need to bound the clique covering of the graph induced by group action, which can be done by bounding the number of disjoint neighbors of orbits. 
However, due to the possibility of the existence of mutual neighbors of orbits, partitioning the graph into $|V|/|\SymG|$ disjoint neighbors of orbits is challenging. 
In particular, for an arbitrary $\delta$-net $V$ for $\ArmSet$, applying $g$ on $x \in V$ does not necessarily result in another point $x' \in V$. 
This leads to the situation where applying the action on two points in the mesh $V$ may end up in a same covering ball, implying that they might not have disjoint neighbors of orbits in $V$. 
To see this, let $x, x' \in V$ such that $\Distance(x,x')< 2\delta$, there could be a case where $g\cdot x, g' \cdot x'$ stay in the same covering ball $\Ball(z,\delta)$ whose diameter is $2\delta$, for some $z \in V$. 
The trivial reduction to $|V|/|\SymG|$ can only be achieved (as in the case of $K$-armed bandits) when we have a special mesh where each point in the mesh can map exactly to another point.
If we only look at the combinatorial properties of the graph induced by group action, it is difficult to prove the number of disjoint dense neighbors of the graph (cliques). 
Instead, we need to look at the geometry of the problem.
As a result of Proposition \ref{Prop: Neighborhood to Clique} and Lemma \ref{lem: upper bound for packing number}, the next lemma shows the upper bound of the clique covering number of the induced graph by utilising the connection between the clique 
covering of the graph and the covering of the fundamental domain.
\end{remark}

\begin{lemma}\label{Lem: upper bound of clique covering}
There are constants $c_1, c_2>0$, such that for any $\delta: 0< \delta < c_2/|\SymG|$, the clique covering number of graph $\Graph_{2\delta}$ is bounded as
\begin{equation}
    \CliqueCoveringNumber(\Graph_{2\delta}) \leq \frac{c_1 \delta^{-d}}{|\SymG|}.
\end{equation}
\end{lemma}


\subsection{\texttt{UniformMesh-N}: Uniform Discretization with the Transformation Group}
\label{subsection: UniformMesh-N and Analysis}

\texttt{UniformMesh-N} is a combination between \texttt{UniformMesh} \cite{Kleinberg2005_UniformMesh} and \texttt{UCB-N} \cite{Caron2012_Sto_UCBN_WithGraph}. 
Roughly speaking, \texttt{UniformMesh-N} first uniformly discretizes the set of arms $\ArmSet$ into a \textit{$\delta$-net} $\Vertices$.
The algorithm then exploits the graph $\Graph_{2\delta}$ induced by the group $\SymG$ as follows. 
In each round $t \in [\TotalRound]$, it chooses an arm $x \in \Vertices$, observes a reward $\Reward_t$ and uses it to update the UCB index of all points within the neighborhood of $x$ in $\Graph_{2\delta}$, that is, $\NeiOrbit(x,2\delta)$.
In this paper, we assume that there is an oracle such that in each round $t$, for a chosen arm $x \in \Vertices$, the oracle returns its neighbor of orbit $\NeiOrbit(x, 2\delta)$.

To formulate the algorithm, we need some additional notations.
For each $x \in \ArmSet$, after round $t \in [\TotalRound]$, let $\PlayedTime(x,t) := \sum_{s=1}^t \bb{I}_{\{\Arm_s = x \}}$ be the number of times the arm $x$ is played, where $\bb{I}$ is the indicator random variable. 
After round $t$, given that the chosen arm is $\Arm_t$, arm $x \in \ArmSet$ is \textit{observed} if and only if $x \in \NeiOrbit(\Arm_t,2\delta)$.
Denote the number of times arm $x$ is observed after the round $t$ as $\ObservedTime(x,t)$, that is, 

\begin{equation}\label{eq: Observed Time in UniformMesh-N}
    \ObservedTime(x,t) := \sum_{s=1}^t \bb{I}_{\{x \in \NeiOrbit(\Arm_s,2\delta) \}}.
\end{equation}

Denote $\ApprExpReward(x,t)$ as empirical mean reward of $x$, that is,

\begin{equation} \label{eq: Empirical mean reward in UniformMesh-N}
    \ApprExpReward(x,t) := \sum_{s=1}^t \Reward_t \bb{I}_{\{ x \in \NeiOrbit(\Arm_s,2\delta) \}}. 
\end{equation}

Let $\UCB(x,t)$ be the UCB index of arm $x$ after round $t$ as follows:

\begin{equation} \label{eq: UCB index in UniformMesh-N}
    \UCB(x,t) := \ApprExpReward(x,t) + \sqrt{\frac{2\log(\TotalRound)}{\ObservedTime(x,t)}} + 3\delta.
\end{equation}

The suboptimal gap of a point $i$ is defined as $\Delta_i = \ExpReward^* - \ExpReward(i)$.
Let $\mathcal{C}$ be a clique covering of graph $\Graph_{2\delta}$.
For each clique $C \in \mathcal{C}$, define its played time as $\PlayedTime(C,t) = \sum_{i\in C} \PlayedTime(i,t)$. 
It is clear that $\PlayedTime(C,t) \leq \ObservedTime(i,t)$ for all $i \in C$.
The pseudocode of \texttt{UniformMesh-N} is given in Algorithm \ref{alg:UniformMesh-N}.

While Algorithm \ref{alg:UniformMesh-N} provides a general principle, we argue that the algorithm can be carried out efficiently with a carefully designed representation of $V$.
In particular, since most of the computation burden is to search $\NeiOrbit(\Arm_t,2\delta)$, we can use a tree of coverings to represent $V$ as in \cite{Bubeck2011_X_armed}, and implement tree search to find the neighbor of any point $x \in \NeiOrbit(\Arm_t,2\delta)$ within the tree with at most $\mathcal O(d)$ operations.
We give a detailed description of how to represent $V$ and implement the search of $\NeiOrbit(\Arm_t,2\delta)$ in Appendix \ref{Appendix: Suggestion of Design and Implementation of Algorithm}.
In each round, it takes only $\mathcal O(|\SymG|d)$ to find an approximation of $\NeiOrbit(\Arm_t,2\delta)$, and with a more refined design, we believe that one can find the exact $\NeiOrbit(\Arm_t,2\delta)$ with a similar strategy.
Therefore, in contrast to sampling within the fundamental domain, which is computationally intractable, our algorithm is computationally efficient by carefully designing the data structure to facilitate searching for the neighbor of orbits.

\begin{algorithm}[h!]
\caption{\texttt{UniformMesh-N}}
\label{alg:UniformMesh-N}
\begin{algorithmic}
    \State Require $\TotalRound, \Vertices, \delta, \SymG$ 
    \State \textbf{Init} $\UCB(x,0) = \infty, \ObservedTime(x,0) = 0, \ApprExpReward(x,0) =0$, for all $ x \in \Vertices$. 
    \For{$t = 1:\TotalRound$}
        \State Play $\Arm_t = \arg\max_{x \in \Vertices} \UCB(x, t-1)$ with ties broken arbitrarily.
        \State Receive a reward $\Reward_t$. 
        \State Compute $\NeiOrbit(X_t, 2\delta)$. 
        \For{$i \in \NeiOrbit(X_t, 2\delta)$}
            \State Update $\ObservedTime(i,t)$ as (\ref{eq: Observed Time in UniformMesh-N}). 
            \State Update $\ApprExpReward(i,t)$ as (\ref{eq: Empirical mean reward in UniformMesh-N}). 
            \State Update $\UCB(i,t)$ as (\ref{eq: UCB index in UniformMesh-N}).
        \EndFor
    \EndFor
\end{algorithmic}
\end{algorithm}

\begin{theorem}\label{Thm: Regret of UniformMesh-N}
Fix an invariant Lipschitz bandit instance with respect to a finite group $\SymG$. 
There are some numbers $a_1, a_2 >0$ such that if $\TotalRound = \Omega\left(|\SymG|^{2d+2}\right)$, the regret of \texttt{UniformMesh-N} algorithm satisfies
\begin{equation} \label{eq: Regret of UniformMesh-N}
    \Regret_\TotalRound \leq a_1\left( \frac{\log(\TotalRound)}{|\SymG|} \right)^{\frac{1}{d+2}}  \TotalRound^{\frac{d+1}{d+2}} +   a_2 \left(\frac{|\SymG|}{\log(\TotalRound)}\right)^{\frac{d}{d+2}} \TotalRound^{\frac{d}{d+2}}.
\end{equation}
\end{theorem}

Note that the second term on the right-hand side of (\ref{eq: Regret of UniformMesh-N}) is insignificant, as it grows at a slower rate in $\TotalRound$ compared to that of the first term.
The proof of Theorem \ref{Thm: Regret of UniformMesh-N} is deferred to Appendix \ref{Appendix: Regret Upper Bound}, but we provide the intuition underlying the analysis here.
First, using the analysis of \cite{Caron2012_Sto_UCBN_WithGraph}, the regret bound increases at most $\mathcal{O}\left( \sum_{C \in \mathcal{C}} \log(\TotalRound)\frac{\max_{i \in C} \Delta_i}{\min_{i\in C} \Delta_i^2}\right)$.
In addition, we already proved that the clique covering number of the graph $\Graph_{2\delta}$ is at most $|\Vertices|/|\SymG|$ when $\TotalRound$ is sufficiently large.
Moreover, for a \say{strongly suboptimal} clique, the factor $\frac{\max_{i \in C} \Delta_i}{\min_{i\in C} \Delta_i^2}$ can be reduced to $\mathcal{O}\left(\delta^{-1}\right)$ since the difference in the expected reward of nodes in any clique cannot be greater than $\delta$.
Putting things together and choosing the suitable value for $\delta$, we obtain Theorem \ref{Thm: Regret of UniformMesh-N}.

\section{Regret Lower Bound}
\label{section: Regret Lower Bound}

This section presents a minimax lower bound for the ILB class that matchs the upper bound in Theorem \ref{Thm: Regret of UniformMesh-N} up to a logarithmic factor, hence show that \texttt{UniformMesh-N} algorithm is near optimal in the minimax sense.
The lower bound for ILB class can be stated formally as follows.

\begin{theorem} \label{Thm: Minimax Lower Bound Invariant Lipschitz}
Consider the invariant Lipschitz bandit problems in $(\ArmSet,\Distance)$ with respect to action of group $\SymG$. 
Suppose $\TotalRound = \Omega(|\SymG|^{d+1})$, then any bandit strategy must suffer a regret at least $\Omega\left( \left(\frac{1}{|\SymG|}\right)^{\frac{1}{d+2}} \TotalRound^{\frac{d+1}{d+2}} \right)$.
\end{theorem}

While the full proof of Theorem \ref{Thm: Minimax Lower Bound Invariant Lipschitz} is deferred to Appendix \ref{Appendix: Regret Lower Bound}, we brief the main idea for the lower bound here. 
First, we prove that the (strictly) packing number of the fundamental domain $\DirichletFD$ is at least $\Omega(\delta^{-d}/|\SymG|)$ when $\delta$ is small enough. 
Then, let $W$ be a strictly packing points in $\DirichletFD$ that has maximum cardinality, we can construct a strictly packing points in $\ArmSet$ using the group action on those points in $\DirichletFD$.
Second, we construct a class of invariant problem instances that assigns the same expected reward for the image of each point in $W$ under the action of $\SymG$, and we show that this class has at least $\Omega(\delta^{-d}/|\SymG|)$ instances.
Applying standard information-theoretic analysis we obtain the lower bound.

\section{Conclusion}
\label{sec: discussion}

In this paper, we consider the class of invariant Lipschitz bandit problems.
We first introduce the idea of the graph induced by the group action, and prove that the clique covering number of this graph is only at most as $|\SymG|^{-1}$ large as the number of vertices.
Based on the concept of the group-induced graph, we propose an algorithm called \texttt{UniformMesh-N}, which uses group orbits as side observations.
Furthermore, using side observation-based analysis, we provide an upper regret for \texttt{UniformMesh-N}, which shows that the improvement by a factor depends on the cardinality of the group $\SymG$; therefore, we are the first to show a strict statistical benefit of using group information to design algorithms.
Finally, we also provide the regret lower bound for the ILB class, which essentially matches the upper bound up to a logarithmic factor.
This implies that \texttt{UniformMesh-N} algorithm is near-optimal in the worst case within the class of ILB problems.
Note that as ILB has many important real-world applications (e.g., matrix factorization bandits, online dictionary learning) our results will open new research directions in those areas as well, providing new ideas to further improve the regret bounds of those respective domains. 

As our current results  use of uniform discretization instead of adaptive discretization, we conjecture that there is still room for improvement.
In particular, as shown in \cite{Kleinberg2019_MetricBandit_Zooming} and \cite{Bubeck2011_X_armed}, uniform discretization is less effective compared to adaptive discretization. 
We believe that by applying group orbit as side observation in adaptive discretization algorithms, one can achieve the regret bound whose \say{shape} is similar to that of Theorem \ref{Thm: Regret of UniformMesh-N}, except that the covering dimension $d$ is replaced by the near-optimal dimension $d'$ \cite{Bubeck2011_X_armed}.
However, the combination of side observations with adaptive discretization algorithms is significantly more challenging in terms of analysis. 
The reason is twofold:
(i) While the analysis technique of side-observation setting requires the set of arms (i.e., feedback graph's vertices) to be fixed over time, the set of candidate arms in adaptive discretization algorithms vary each round; 
(ii) The proof of Lemma \ref{lem: upper bound for packing number} heavily depends on the nice geometric properties of the Dirichlet domain, particularly, the set is full dimension with zero volume boundary. 
    In contrast, a near-optimal set (e.g., the near-optimal set assumption in \cite{Bubeck2011_X_armed}) can have an arbitrary shape.
    Therefore, without any further assumption, it is difficult to prove the upper bound for the covering number of the part of near-optimal set lying in the fundamental domain.

\color{black}

%
%
%
%
\newpage
\bibliographystyle{plainnat}
\bibliography{ILBRef}

\newpage
\appendix
\onecolumn
\begin{center}
{\huge Appendix}
\end{center}

\section{Properties of The Graph Induced by the Group Action}
\label{Appendix: Graph Proof}

\subsection{Proof of Lemma \ref{lem: upper bound for packing number}}
To prove Lemma \ref{lem: upper bound for packing number}, we first need the following proposition.

\begin{proposition} \label{prop: Vol of set near boundary}
Let $H$ be a $(d-1)$-dimensional hyperplane in an Euclidean space $\Euclidean^d, d>1$, and $S \subset H \subset \Euclidean^d$ be a bounded subset. For $\delta > 0$, define $K = \{x \in \Euclidean^d \mid \Distance(x,S) < \delta \}$. For a small positive number $0 < \delta < \frac{1}{4}$, there exists a finite positive constant $c>0$ such that $\Vol(K) \leq c \delta $.
\end{proposition}

\begin{proof}
By rotation and translation, we can suppose that $H$ is the vector subspace whose last coordinate is $0$. 
Then $S$ is covered by a fixed cube in $H$, that is,
\[\left(\prod_{i=1}^{d-1}I_i \right)\times \{0\},\] 
where $I_i = [a_i, b_i]$ for some finite numbers $a_i, b_i $ such that $a_i < b_i$.
Then $K$ is a subset of the whole cube $\left(\prod_{i=1}^{d-1}I_i \right)\times [-\delta, \delta]$. 
Define $l_i = b_i - a_i$, we have
\begin{equation*}
    \begin{aligned}
        \Vol(K) &\leq 2\delta \prod_{i=1}^{d-1} (l_i + 2\delta) \\
        &\leq 2 \delta c_{\mathrm{max}} \sum_{i=0}^{d-1}   (2\delta)^{i}  \\
        &\leq 2  \delta c_{\mathrm{max}} \frac{1}{1-2\delta},\\
    \end{aligned}
\end{equation*}
for some constant $c_{\mathrm{max}} > 0$.
Now, for $0< \delta < 1/4$, we have $\Vol(K) \leq   4  c_{\mathrm{max}} \delta < c\delta$ for some number $c > 0$ that is independent of $\delta$.
\end{proof}
Now we proceed to the proof of Lemma \ref{lem: upper bound for packing number}.

\begin{proof}[Proof of Lemma \ref{lem: upper bound for packing number}]
Denote $\CoveringDirichletFDInside = \{x \in \CoveringDirichletFD \mid \Distance(x,\partial \DirichletFD) \leq \delta \}$ and $\CoveringDirichletFDBoundary = \{x \in \CoveringDirichletFD \mid \Distance(x,\partial \DirichletFD) > \delta \}$.
Thus, $\CoveringDirichletFDInside$ and $\CoveringDirichletFDBoundary$ are disjoint subsets of $\CoveringDirichletFD$ whose union is $\CoveringDirichletFD$. 
Now, we prove the upper bounds for the cardinality of $\CoveringDirichletFDInside$ and $\CoveringDirichletFDBoundary$.

\paragraph{First Step.} By the construction of the Dirichlet domain $\DirichletFD$, its boundary is a finite union of bounded subsets in some $(d-1)$-dimensional hyperplanes. In particular,
\begin{equation}
    \partial \DirichletFD \subseteq \bigcup_{i = 1}^m S_i,
\end{equation}
where $S_i$ is a bounded subset in some $(d-1)$-dimensional hyperplane, and $m \leq |\SymG|$.

Now, we bound the number of elements in $\CoveringDirichletFD$ whose distance to $\partial \DirichletFD$ is smaller than $\delta$. 
Denote the set of points whose distance near the boundary is $B$, that is, $B = \{x \in \ArmSet \mid \Distance(x, S_i) \leq \delta, \forall i\in [m]  \}$.
Denote $K_i = \{ x \in \ArmSet \mid \Distance(x, S_i) \leq \delta \}$, we have

\begin{equation*}
    B = \bigcup_{i=1}^m K_i.
\end{equation*}

This implies $\Vol(B) \leq \sum_{i=1}^m \Vol(K_i) $. 
By Proposition \ref{prop: Vol of set near boundary}, it follows that 

\[\Vol(B) \leq \mathcal{O}(\delta).\]

\paragraph{Second step.}  By Proposition \ref{prop: LB and UB for covering and packing number}, it follows that
 
\begin{equation}
    \PackingNumber(B,\Distance,\delta) \leq \left(\frac{3}{\delta} \right)^{d} \frac{\Vol(K)}{\Vol(\Ball)} \leq c_2 \delta^{-(d-1)},
\end{equation}

where $c_2$ is some positive number and $\Ball$ is the unit ball of $\Euclidean^d$. 
Since $\CoveringDirichletFDInside$ is a $\delta$-packing of $B$, it follows that 

\begin{equation} \label{eq: temp near boundary point}
    |\CoveringDirichletFDInside| \leq \PackingNumber(B,\Distance,\delta)  \leq c_2 \delta^{-(d-1)}
\end{equation}

\paragraph{Third step.} Since there are $|\SymG|$ disjoint images of $\DirichletFD$, if we allocate $|\CoveringDirichletFD|$ points to such 
images, there must be one image containing at most $|\Vertices| /|\SymG|$ points. 
Without loss of generality, let $\DirichletFD$ be such image, we obtain 

\begin{equation} \label{eq: temp strictly inner point}
\begin{aligned}
      |\CoveringDirichletFDBoundary| &\leq |\Vertices| /|\SymG|
      & \leq \frac{c_1 \delta^{-d}}{|\SymG|},
\end{aligned}
\end{equation}

for some constant $c_1>0$, where the last inequality holds by Proposition \ref{prop: LB and UB for covering and packing number}.
From (\ref{eq: temp near boundary point}) and (\ref{eq: temp strictly inner point}), it follows that
\begin{equation*}
\begin{aligned}
        |\CoveringDirichletFD| &\leq \frac{c_1 \delta^{-d}}{|\SymG|} + c_2 \delta^{-(d-1)}
        & = \frac{\delta^{-d} (c_1 + c_2 |\SymG| \delta)}{\SymG}
        \leq \frac{c_4 \delta^{-d}}{|\SymG|},
\end{aligned}
\end{equation*}
where the last inequality holds for some numbers $c_3, c_4 >0$ and $\delta < \frac{c_3}{|\SymG|}$.
\end{proof}

\subsection{Proof of Lemma \ref{Lem: upper bound of clique covering}}

\begin{proof}[Proof of Proposition \ref{Prop: Undirected edge}]
Let $\SymG \leq \SymmetryGroup{\ArmSet}$ be a finite subgroup. 
For any $x,x' \in \ArmSet$ and a constant $\delta >0$, if there is an action $g \in \SymG$ such that $\Distance(g\cdot x, x') <\delta$, then $\Distance(x, g^{-1}\cdot x') <\delta$.
\end{proof}

\begin{proof}[Proof of Proposition \ref{Prop: Neighborhood to Clique}]
If $\NeiOrbit(i,\delta) = \{i \}$ then it is a clique. Now, assume that $\NeiOrbit(i,\delta)$ consists of at least 2 points of $\Vertices$.
Let $j,j' \in \NeiOrbit(i,\delta)$, that is, there exists $g_1, g_2 \in \SymG$ such that

\begin{equation}
    \begin{cases}
        \Distance(j, g_1 \cdot i) < \delta , \\
        \Distance(j',g_2 \cdot i) < \delta .
    \end{cases}
\end{equation}

Applying the inverse action $g_1^{-1}, g_2^{-1}$ to the left-hand side of the above equations, given the fact that $g_1^{-1}, g_2^{-1}$ are isometries, we have

\begin{equation}
    \begin{cases}
        \Distance(g_1^{-1} \cdot j,i) < \delta \\
        \Distance(g_2^{-1} \cdot j',i) < \delta.
    \end{cases}
\end{equation}

Therefore, we obtain the following: 

\begin{equation*}
\begin{aligned}
    \Distance(j,(g_1\cdot g_2^{-1} \cdot j') )
    &= \Distance(g_1^{-1} \cdot j,g_2^{-1} \cdot j) \\
    & \leq  \Distance(g_1^{-1} \cdot j,i) + \Distance(g_2^{-1} \cdot j,i) < 2\delta.
\end{aligned}
\end{equation*}

Therefore, $(j,j') \in \Edges_{2\delta}$. As this holds for any arbitrary pair in $\NeiOrbit(i,\delta)$, it follows that $\NeiOrbit(i,\delta)$ is a clique of $\Graph_{2\delta}$.
\end{proof}

\begin{proof}[Proof of Lemma \ref{Lem: upper bound of clique covering}]
As $\DirichletFD$ is a fundamental domain for the group $\SymG$ in $\ArmSet$, for any $x \in \ArmSet$, there is a group element $g\in \SymG$ such that $g\cdot x \in \overline{\DirichletFD}$.
Besides, since $\CoveringDirichletFD$ is \textit{$\delta$-net} of $\overline{\DirichletFD}$, there must exist $i \in \CoveringDirichletFD$ such that $\Distance(i,g\cdot x) < \delta$, that is, $x \in \NeiOrbit(i,\delta)$.
Therefore, we have

\[\Vertices = \bigcup_{i\in \CoveringDirichletFD}\NeiOrbit(i,\delta).\]

Now, for each $i\in \Vertices$, by Lemma \ref{Prop: Neighborhood to Clique}, $\NeiOrbit(i,\delta)$ forms a clique in $\Graph_{2\delta}$. 
Therefore, by Lemma \ref{lem: upper bound for packing number}, there are some constants $c_1>0, c_2>0$ such that for $0<\delta< c_2/|\SymG|$, we have

\[\CliqueCoveringNumber(\Graph_{2\delta}) \leq |\CoveringDirichletFD| \leq \frac{c_1\delta^{-d}}{|\SymG|}.\]
\end{proof}

\section{Regret Upper Bound of \texttt{UniformMesh-N} algorithm}
\label{Appendix: Regret Upper Bound}

To prove Theorem \ref{Thm: Regret of UniformMesh-N}, we first begin with the following lemmas.
\begin{lemma} \label{lem: UB of prob bad event 1}
Consider the optimal node $x$ in $\Vertices$. For all $t \in [\TotalRound]$, we have 
\begin{equation}
    \bb{P} \left\{\UCB(x,t) \leq \ExpReward^* \right\} \leq t^{-3}.
\end{equation}
\end{lemma}

\begin{proof}
Let $i^* = \arg \max_{i\in \Vertices} \ExpReward(i)$ be an optimal node. Since $\Vertices$ is a covering, under the Lipschitz assumption, we have $\ExpReward^* \leq \ExpReward(i^*) + \delta$. 
As the lemma holds trivially true when $\ObservedTime(i^*,t) = 0$, we only consider the case where $\ObservedTime(i^*,t) \geq 1$. We have

\begin{equation}
    \begin{aligned}
        \bb{P} & \left\{ \UCB(i^*,t) \leq \ExpReward^* \text{ and }  \ObservedTime(i^*,t) \geq 1 \right\} \\
        &\leq \bb{P} \left\{ \UCB(i^*,t) \leq \ExpReward(i^*) + \delta  \right\} \\
        &= \bb{P}\left\{ \ApprExpReward(i^*,t) + \sqrt{\frac{2 \log(\TotalRound)}{\ObservedTime(i^*,t)}} + 3\delta  \leq \ExpReward(i^*) + \delta  \right\} \\
        &= \bb{P}\left\{ \ObservedTime(i^*,t)\ApprExpReward(i^*,t) + \ObservedTime(i^*,t)(2\delta - \ExpReward(i^*)) \leq -\sqrt{2\ObservedTime(i^*,t) \log n } \right\} \\
        &= \bb{P}\Bigg\{\sum_{s=1}^t (\Reward_s - \ExpReward(\Arm_s))\bb{I}_{\{i^* \in \NeiOrbit(\Arm_s,2\delta) \}} +  \sum_{s=1}^t (\ExpReward(\Arm_s) + 2\delta - \ExpReward(i^*))\bb{I}_{\{i^* \in \NeiOrbit(\Arm_s,2\delta) \}} \\
        & \quad \quad \leq -\sqrt{2\ObservedTime(i^*,t) \log n }  \Bigg\} \\
        &\leq \bb{P}\left\{\sum_{s=1}^t (\Reward_s - \ExpReward(\Arm_s))\bb{I}_{\{i^* \in \NeiOrbit(\Arm_s,2\delta) \}} \leq -\sqrt{2\ObservedTime(i^*,t) \log n }  \text{  and  }  \ObservedTime(i^*t) \leq 1 \right\},
    \end{aligned}
\end{equation}

where the last inequality is due to $\ExpReward(\Arm_s) + 2\delta - \ExpReward(i^*) >0$ if $i^* \in \NeiOrbit(\Arm_s,2\delta)$, as (\ref{eq: Bound for difference of mean reward in neighborhood}).
Now, let $\ObservedTime_j = \min\{t: \ObservedTime(i,t) = j\}$ for $j=1,2,\cdots$, and denote $\Tilde{\Arm}_j = \Arm_{\ObservedTime_j}$ and $\Tilde{\Reward}_j = \Reward_{\ObservedTime_j}$.
Now, applying Hoeffding's inequality and the union bound for $1 < \ObservedTime(i^*,t) \leq t$, we obtain the following:

\begin{equation}
    \begin{aligned}
        \bb{P} & \left\{ \sum_{s=1}^t (\Reward_s - \ExpReward(\Arm_s)) \bb{I}_{\{i^* \in \NeiOrbit(\Arm_s,2\delta) \}} \leq -\sqrt{2\ObservedTime(i^*,t) \log(\TotalRound) } \right\} \\
        &= \bb{P} \left\{ \sum_{j=1}^{\ObservedTime(i^*,t)} \left(\Tilde{\Reward}_j - \ExpReward(\Tilde{\Arm}_j) \right)  \leq - \sqrt{2 \ObservedTime(i^*,t) \log(\TotalRound)}   \right\}. \\
    \end{aligned}
\end{equation}

Applying the Hoeffding-Azuma inequality for $ \left\{\sum_{j=1}^{\ObservedTime(i^*,t)} (\Tilde{Y}_j - \ExpReward(\Tilde{\Arm}_j)) >  - \sqrt{2 \ObservedTime(i^*,t) \log(\TotalRound)} \right\}$ and taking the union bound for $1 \leq \ObservedTime(i^*,t) \leq t$, we obtain the following.

\begin{equation}
    \begin{aligned}
        \bb{P} &\left\{ \sum_{j=1}^{\ObservedTime(i^*,t)} \left(\Tilde{\Reward}_j - \ExpReward(\Tilde{\Arm}_j) \right)  \leq - \sqrt{2 \ObservedTime(i^*,t) \log(\TotalRound)}   \right\} \\
        &\leq \sum_{\tau=1}^t \bb{P} \left\{\sum_{j = 1 }^\tau \left(\Tilde{\Reward}_j - \ExpReward(\Tilde{\Arm}_j) \right)  \leq -\sqrt{2 \tau \log(\TotalRound)}   \right\} \\
        &\leq \sum_{\tau=1}^t \exp{\left(-2 \frac{2 \tau \log(\TotalRound)}{ \tau} \right )} \\
        &\leq \sum_{\tau=1}^t \TotalRound^{-4} \leq \TotalRound^{-3} \leq t^{-3}.
    \end{aligned}
\end{equation}

\end{proof}

\begin{lemma} \label{lem: UB of prob bad event 2}
Consider the clique $C \in \mathcal{C}$ such that $\min_{i \in C} \Delta_i > 5\delta$. 
For all $t \in [\TotalRound]$, and for an integer $l_c \geq 1$ such that
\begin{equation} \label{eq: lb of lc}
    l_c \geq \frac{8 \log(\TotalRound)}{\left(\min_{i \in C} \Delta_i - 5\delta \right)^2},
\end{equation}
for all $i \in C$, we have
\[\bb{P}\left\{ \UCB(i,t) > \ExpReward^* \text{ and } \PlayedTime(C,t) > l_c \right\} \leq t \TotalRound^{-4}. \]
\end{lemma}

\begin{proof}
From (\ref{eq: lb of lc}), we obtain the following.

\begin{equation}\label{eq: lb for lc all i}
    \begin{cases}
        l_c \geq \frac{8 \log(\TotalRound)}{(\Delta_i - 5\delta)^2} & \text{$\forall i \in C$} \\
        \ObservedTime(i,t) \geq \PlayedTime(C,t) \geq l_c.
    \end{cases}
\end{equation}

Therefore, we have 

\begin{equation}
    \begin{aligned}
        \sqrt{\frac{2\log(\TotalRound)}{\ObservedTime(i,t)}} &\leq \sqrt{\frac{2\log(\TotalRound)}{l_c}} \leq \frac{\Delta_i - 5\delta}{2}.
    \end{aligned}
\end{equation}

Now, consider

\begin{equation}
    \begin{aligned}
        \bb{P} & \left\{ \UCB(i,t) > \ExpReward^* \text{ and }  \PlayedTime(C,t) > l_c \right\} \\
        &= \bb{P}\left\{ \ApprExpReward(i,t) + \sqrt{\frac{2 \log(\TotalRound)}{\ObservedTime(i,t)}} + 3\delta > \ExpReward(i) + \Delta_i   \text{ and }  \PlayedTime(C,t) > l_c \right\} \\
        &\leq \bb{P}\left\{ \ApprExpReward(i,t) +  \frac{\Delta_i - 5\delta}{2} + 3\delta > \ExpReward(i) + \Delta_i  \text{ and }  \PlayedTime(C,t) > l_c \right\} \\
        &= \bb{P}\Bigg\{ \sum_{s=1}^t \left(\Reward_s - \ExpReward(\Arm_s) \right) \bb{I}_{\{i \in \NeiOrbit(\Arm_s,2\delta)\}} + \ObservedTime(i,t)(\ExpReward(\Arm_s) - 2\delta - \ExpReward(i)) \\
        & \quad \quad > \ObservedTime(i,t) \frac{\Delta_i - 5\delta}{2} \text{ and }  \PlayedTime(C,t) > l_c \Bigg\} \\
        &\leq \bb{P}\left\{\sum_{s=1}^t \left(\Reward_s - \ExpReward(\Arm_s) \right) \bb{I}_{\{i \in \NeiOrbit(\Arm_s,2\delta)\}} > \ObservedTime(i,t) \frac{\Delta_i - 5\delta}{2}   \text{ and }  \PlayedTime(C,t) > l_c \right\}. \\
    \end{aligned}
\end{equation}

Where the last inequality holds since $\ExpReward(\Arm_s) - 2\delta - \ExpReward(i) \leq 0$ if $i \in \NeiOrbit(\Arm_s,2\delta)$, as (\ref{eq: Bound for difference of mean reward in neighborhood}). 
Now, let $\ObservedTime_j = \min\{t: \ObservedTime(i,t) = j\}$ for $j=1,2,\cdots$, and denote $\Tilde{\Arm}_j = \Arm_{\ObservedTime_j}$ and $\Tilde{\Reward}_j = \Reward_{\ObservedTime_j}$. We have

\begin{equation}
    \begin{aligned}
        \bb{P} &\left\{\sum_{s=1}^t \left(\Reward_s - \ExpReward(\Arm_s) \right) \bb{I}_{\{i \in \NeiOrbit(\Arm_s,2\delta)\}} > \ObservedTime(i,t) \frac{\Delta_i - 5\delta}{2}   \text{ and }  \PlayedTime(C,t) > l_c \right\} \\
        &= \bb{P}\left\{\sum_{j = 1 }^{\ObservedTime(i,t)} \left(\Tilde{\Reward}_j - \ExpReward(\Tilde{\Arm}_j) \right) > \ObservedTime(i,t) \frac{\Delta_i - 5\delta}{2}   \text{ and }  \PlayedTime(C,t) > l_c \right\}. \\
    \end{aligned}
\end{equation}

Applying the Hoeffding-Azuma inequality for $ \left\{\sum_{j=1}^{\ObservedTime(i,t)} (\Tilde{Y}_j - \ExpReward(\Tilde{\Arm}_j)) > \ObservedTime(i,t) \frac{\Delta_i - 5\delta}{2} \right\}$ and taking the union bound for $l_c < \ObservedTime(i,t) \leq t$, we obtain the following.

\begin{equation}
    \begin{aligned}
        \bb{P} &\left\{\sum_{j = 1 }^{\ObservedTime(i,t)} \left(\Tilde{\Reward}_j - \ExpReward(\Tilde{\Arm}_j) \right) > \ObservedTime(i,t) \frac{\Delta_i - 5\delta}{2}   \text{ and }  \PlayedTime(C,t) > l_c \right\} \\ 
        &\leq \sum_{\tau = l_c}^t \bb{P}\left\{\sum_{j = 1}^\tau \left(\Tilde{\Reward}_j - \ExpReward(\Tilde{\Arm}_j) \right) > \tau \frac{\Delta_i - 5\delta}{2}\right\} \\
        &\leq \sum_{\tau = l_c}^t \exp{\left(-\frac{2}{\tau} \left(\tau \frac{\Delta_i - 5\delta}{2} \right)^2 \right)} 
        =  \sum_{\tau = l_c}^t \exp{\left(-\frac{1}{2} \tau (\Delta_i - 5\delta)^2 \right)} \\
        &\leq t \exp{\left(-\frac{1}{2} l_c (\Delta_i - 5\delta)^2 \right)} \\
        &\leq tn^{-4}.
    \end{aligned}
\end{equation}

\end{proof}

\begin{lemma}\label{lem: gap-dependent regret bound}
Suppose that there is a clique covering $\mathcal{C}$ for graph $\Graph_{2\delta}$. 
Let $\mathcal{C}_1 = \left\{C \in \mathcal{C} \mid \min_{i \in C} \Delta_i > 6 \delta  \right\}$. 
Then, the regret of \texttt{UniformMesh-N} satisfies
\begin{equation}
    \Regret_\TotalRound \leq \sum_{C\in \mathcal{C}_1}  \left( 8 \log (\TotalRound) \left( \frac{ \max_{i\in C} \Delta_i  }{(\min_{i\in C} \Delta_i - 5\delta)^2} \right)  \right) + 4 \sum_{i\in \Vertices} \Delta_i +  8\TotalRound \delta.
\end{equation}
\end{lemma}

\begin{proof} 
The proof consists of three steps.
\paragraph{First step.} We need to bound regret of playing a node in a clique $C \in \mathcal{C}_1$.
For a clique $C \in \mathcal{C}_1$, define $\RdRegret_{\TotalRound, C} =   \sum_{t=1}^n \sum_{i\in C} \Delta_i \bb{I}_{\{\Arm_t = i\}}$. We have
\begin{equation}
    \begin{aligned}
        \RdRegret_{\TotalRound, C} &=   \sum_{t=1}^n \sum_{i\in C} \Delta_i \bb{I}_{\{\Arm_t = i\}} \\
        &=  \sum_{i\in C} \Delta_i   \sum_{t=1}^n \bb{I}_{\{\Arm_t = i \text{ and } \left( {\PlayedTime_C(t) \leq l_c } \text{ or } {\PlayedTime_C(t) > l_c } \right) \}}  \\
        &\leq l_c \max_{i\in C} \Delta_i + \sum_{i\in C} \Delta_i \sum_{t=1}^n \bb{I}_{\{\Arm_t = i \text{ and }  {\PlayedTime_C(t) > l_c }  \}} \\
    \end{aligned}
\end{equation}
for an integer $l_c >1$. Consider the event $\{\Arm_t = i \text{ and }  {\PlayedTime_C(t) > l_c }  \} $, it follows that

\begin{equation}
    \begin{aligned}
        \bb{P} & \left\{\Arm_t = i \text{ and }  {\PlayedTime(C,t) > l_c } \right\} \\
        &\leq \bb{P} \left\{ \UCB_i(t) > \UCB_{i^*}(t)  \text{ and }  {\PlayedTime_C(t) > l_c } \right\} \\
        &\leq \bb{P} \left\{\left\{ \UCB_i(t) >  \ExpReward^* \text{ or  } \UCB_{i^*}(t) \leq \ExpReward^* \right\}  \text{ and }  {\PlayedTime_C(t) > l_c } \right\} \\
        &\leq \bb{P} \left\{  \UCB_i(t) > \ExpReward^* \text{ and }  {\PlayedTime_C(t) > l_c }  \right\} + \bb{P} \left\{\UCB_{i^*}(t) \leq \ExpReward^*  \right\} \\
    \end{aligned}
\end{equation}

For the choice $l_c$ 
\[l_c = \frac{8 \log \TotalRound}{(\min_{i \in C} \Delta_i - 5\delta)^2} +1,\]
according to lemmas \ref{lem: UB of prob bad event 1} and \ref{lem: UB of prob bad event 2}, one obtains the upper bound as follows.

\begin{equation}
    \bb{P}  \left\{\Arm_t = i \text{ and }  {\PlayedTime_C(t) > l_c } \right\} \leq \TotalRound t^{-4} + t^{-3}.
\end{equation}

Therefore, the regret of playing a clique $C\in \mathcal{C}_1$ satisfies

\begin{equation}
    \begin{aligned}
        \Regret_{\TotalRound, C} &= \bb{E}\left[\RdRegret_{\TotalRound, C} \right] \\
        &\leq \left( \frac{8 \log \TotalRound}{(\min_{i \in C} \Delta_i} - 5\delta)^2 +1 \right) \max_{i\in C} \Delta_i + \sum_{i\in C}\Delta_i \sum_{t=1}^n (\TotalRound t^{-4} + t^{-3}) \\
        &\leq \left( \frac{8 \log \TotalRound}{(\min_{i \in C} \Delta_i - 5\delta)^2} +1 \right) \max_{i\in C} \Delta_i + 3 \sum_{i\in C}\Delta_i \\
        &\leq \left( \frac{8 \log \TotalRound}{(\min_{i \in C} \Delta_i - 5\delta)^2} \right) \max_{i\in C} \Delta_i + 4 \sum_{i\in C}\Delta_i.
    \end{aligned}
\end{equation}

Since we have

\begin{equation*}
\begin{aligned}
        \sum_{t=1}^n (\TotalRound t^{-4} + t^{-3}) &\leq \sum_{t=1}^\TotalRound \TotalRound^{-3} + \sum_{t=1}^\TotalRound t^{-2} \\
        &  \leq \TotalRound^{-2} + \sum_{t=1}^\infty t^{-2}  \leq 1 + \frac{\pi^2}{6} < 3.
\end{aligned}
\end{equation*}

\paragraph{Second step.} Next, we bound the regret when playing a node not in $\mathcal{C}_1$. Consider a clique $C$ such that $\min_{i\in C}\Delta_i \leq 6\delta$. 
Since (\ref{eq: Bound for difference of mean reward in neighborhood}) holds for any neighborhood of the graph $\Graph_{2\delta}$, for all $i \in C$, we have

\[\Delta_i \leq \min_{i\in C}\Delta_i + 2\delta \leq 8\delta.\]

Therefore, the regret of playing a node in $C$ is simply upper-bounded by $8\delta \TotalRound$.

\paragraph{Third step.} Taking the sum of regret for playing all the cliques in $\mathcal{C}$, note that the cliques in $\mathcal{C}$ are disjoint. 
We obtain the following:
\begin{equation*}
    \begin{aligned}
        \Regret_\TotalRound &= \sum_{C \in \mathcal{C}_1} \left( \left( \frac{8 \log (\TotalRound) }{(\min_{i\in C} \Delta_i - 5\delta)^2} \right) \max_{i\in C} \Delta_i + 4 \sum_{i\in C} \Delta_i \right)  +  8\TotalRound \delta \\
         &= \sum_{C\in \mathcal{C}_1}  \left( \left( \frac{8 \log (\TotalRound) }{(\min_{i\in C} \Delta_i - 5\delta)^2} \right) \max_{i\in C} \Delta_i  \right) + 4 \sum_{i\in \Vertices} \Delta_i +  8\TotalRound \delta 
    \end{aligned}.
\end{equation*}

\end{proof}

\begin{lemma}\label{lem: gap-independent regret bound - standard trick}
For any constant $\delta >0$, for a clique covering $\mathcal{C}$ of $\Graph_{2\delta}$, the regret of \texttt{UniformMesh-N} satisfies
\begin{equation} \label{eq: gap-independent regret bound - standard trick}
    \Regret_\TotalRound \leq |\mathcal{C}| \frac{64 \log(\TotalRound)}{ \delta}  + 4 \sum_{i\in \Vertices} \Delta_i +  8\TotalRound \delta.
\end{equation}
\end{lemma}

\begin{proof}
Let $\mathcal{C}_1 = \left\{C \in \mathcal{C} \mid \min_{i \in C} \Delta_i > 6 \delta  \right\}$. For all $C\in \mathcal{C}_1$, we have $\min_{i \in C} \Delta_i - 5\delta > \delta$, and $\max_{i \in C} \Delta_i < \min_{i\in C} \Delta_i + 2\delta$. By Lemma \ref{lem: gap-dependent regret bound}, we obtain

\begin{equation*}
    \begin{aligned}
    \Regret_\TotalRound &\leq \sum_{C\in \mathcal{C}_1}  \left( \left( \frac{8 \log (\TotalRound) }{(\min_{i\in C} \Delta_i - 5\delta)^2} \right) \max_{i\in C} \Delta_i  \right) + 4 \sum_{i\in \Vertices} \Delta_i +  8\TotalRound \delta \\
    & \leq 8\log(\TotalRound) \sum_{C \in \mathcal{C}_1} \left(\frac{\min_{i\in C} \Delta_i - 5\delta}{(\min_{i\in C} \Delta_i - 5\delta)^2} +  \frac{7\delta}{(\min_{i\in C} \Delta_i - 5\delta)^2} \right) + 4 \sum_{i\in \Vertices} \Delta_i +  8\TotalRound \delta  \\
    & = 8\log(\TotalRound) \sum_{C \in \mathcal{C}_1} \left(\frac{1}{\min_{i\in C} \Delta_i - 5\delta} +  \frac{7\delta}{(\min_{i\in C} \Delta_i - 5\delta)^2} \right) + 4 \sum_{i\in \Vertices} \Delta_i +  8\TotalRound \delta \\
    & \leq 8\log(\TotalRound) \sum_{C \in \mathcal{C}_1} \left(\frac{1}{\delta} +  \frac{7\delta}{\delta^2} \right) + 4 \sum_{i\in \Vertices} \Delta_i +  8\TotalRound \delta \\
    & \leq |\mathcal{C}| \frac{64 \log(\TotalRound)}{\delta}  + 4 \sum_{i\in \Vertices} \Delta_i +  8\TotalRound \delta. \\
    \end{aligned}.
\end{equation*}

\end{proof}

Now, We we proceed to the proof of Theorem \ref{Thm: Regret of UniformMesh-N}.
\begin{proof}[Proof of Theorem \ref{Thm: Regret of UniformMesh-N}]
Since Lemma \ref{lem: gap-independent regret bound - standard trick} holds for any clique covering of graph $\Graph_{2\delta}$, choose a clique $\mathcal{C}$ such that $|\mathcal{C}| = \CliqueCoveringNumber(\Graph_{2\delta})$. By Lemma \ref{Lem: upper bound of clique covering}, we have that for some constants $c_1, c_2 > 0$ and $0< \delta < c_2/|\SymG|$, 

\[|\mathcal{C}| \leq \frac{c_1\delta^{-d}}{|\SymG|}.\]

Plugging this into (\ref{eq: gap-independent regret bound - standard trick}), we obtain

\begin{equation} \label{eq: inter 1}
\begin{aligned}
    \Regret_\TotalRound &\leq |\mathcal{C}| \frac{64 \log(\TotalRound)}{ \delta}   +  8\TotalRound \delta  + 4 \sum_{i\in \Vertices} \Delta_i \\
    &\leq c_3 \frac{ \log(\TotalRound) \delta^{-(d+1)}}{|\SymG|} + c_4 \TotalRound \delta + c_5\delta^{-d},
\end{aligned}
\end{equation}
for some constants $c_3, c_4, c_5 >0$, where the third term in the RHS of the second inequality is due to $|\Vertices| = \Theta(\delta^{-d})$.

Choosing
\begin{equation}\label{eq: UB - choice of delta}
    \delta = \left(\frac{\log(\TotalRound) }{\TotalRound |\SymG|}\right)^{\frac{1}{d+2}}.
\end{equation}

Since $\delta < c_2/|\SymG|$, we have 

\begin{equation}\label{eq: inter 2}
    \frac{\TotalRound}{\log(\TotalRound)} > c_6 |\SymG|^{d+1},
\end{equation}

for some number $c_6>0$. 
Moreover, as $\sqrt{\TotalRound} > \log(\TotalRound)$ for $\TotalRound>1$, (\ref{eq: inter 2}) is satisfied if we choose

\[ \TotalRound > c_7 |\SymG|^{2d+2}, \]

for a number $c_7>0$. 
Now, plugging the choice of $\delta$ as (\ref{eq: UB - choice of delta}) into (\ref{eq: inter 1}), we obtain

\begin{equation*}
\begin{aligned}
    \Regret_\TotalRound \leq a_1\left( \frac{\log(\TotalRound)}{|\SymG|} \right)^{\frac{1}{d+2}}  \TotalRound^{\frac{d+1}{d+2}} +   a_2 \left(\frac{|\SymG|}{\log(\TotalRound)}\right)^{\frac{d}{d+2}} \TotalRound^{\frac{d}{d+2}},
\end{aligned}
\end{equation*}
for some constants $a_1, a_2 > 0$.  

\end{proof}

\section{Regret Lower Bound for ILB}
\label{Appendix: Regret Lower Bound}

This section presents a lower bound for the invariant Lipschitz bandit class.
Let $\PackingPointsFD$ be a $\delta$-packing of the Dirichlet domain $\DirichletFD$ whose size is $\PackingNumber(\DirichletFD,\Distance,\delta)$, and $\PackingPointsFarFromBoundaryFD \subseteq \PackingPointsFD$ be a strictly $\delta$-packing of $\DirichletFD$. 
The following result is similar to Lemma \ref{lem: upper bound for packing number}.

\begin{lemma} \label{lem: LB for Packing Points Far FromBoundaryFD}
Given a Dirichlet domain $\DirichletFD$ for a finite group $\SymG$, there are constants $c_1, c_2 >0$ such that

\begin{equation}
    |\PackingPointsFarFromBoundaryFD| \geq \frac{c_1 \delta^{-d}}{|\SymG|} - c_2 \delta^{-(d-1)}.
\end{equation}

Moreover, there are numbers $c_3, c_4 >0$ such that for $0< \delta \leq c_3/|\SymG|$, we have

\begin{equation}
    |\PackingPointsFarFromBoundaryFD| \geq \frac{c_4 \delta^{-d}}{|\SymG|}.
\end{equation}

\end{lemma}

\begin{proof}
The proof consists of 2 steps.
\paragraph{Step 1.} Suppose that $\Vertices$ is a $\delta$-packing  of $\ArmSet$ whose cardinality is $\PackingNumber(\ArmSet,\Distance,\delta)$.
Now, if we allocate $\Vertices$ to $|\SymG|$ disjoint images of $\DirichletFD$, denote $\Vertices_{g \cdot \DirichletFD}^p = \{ \Vertices \cap g\cdot \DirichletFD \}$, it follows that

\begin{equation*}
    \max_{g\in \SymG} |\Vertices_{g \cdot \DirichletFD}^p | \geq \frac{\PackingNumber(\ArmSet,\Distance,\delta)}{|\SymG|}.
\end{equation*}

Without loss of generality, let the identity group element maximize $|\Vertices_{g \cdot \DirichletFD}^p |$.
As $\Vertices_{\DirichletFD}^p$ is a set of packing points in $\DirichletFD$, it follows that for some number $c_1 >0$, we have

\begin{equation*}
    |\PackingPointsFD| = \PackingNumber(\DirichletFD, \Distance,\delta) 
    \geq \Vertices_{\DirichletFD}^p \geq \frac{\PackingNumber(\ArmSet,\Distance,\delta)}{|\SymG|} \geq \frac{c_1 \delta^{-d}}{|\SymG|}.
\end{equation*}

\paragraph{Step 2.} By Lemma \ref{lem: upper bound for packing number}, the number of points in $\PackingPointsFD$ whose distance to the boundary of the Dirichlet domain is smaller than $\delta$ is bounded by $O(\delta^{d-1})$. 
Formally, $\left|\{x \in \PackingPointsFD \mid \Distance(x, \partial \DirichletFD) \leq \delta \} \right| \leq c_2 \delta^{-(d-1) }$ for a number $c_2>0$. It follows that 

\begin{equation} \label{eq: def of reward function in LB}
\begin{aligned}
    |\PackingPointsFarFromBoundaryFD| &\geq |\PackingPointsFD| -  c_2 \delta^{-(d-1)} \\
    &\geq \frac{c_1 \delta^{-d}}{|\SymG|} - c_2 \delta^{-(d-1)}.
\end{aligned}
\end{equation}

Moreover, there are some numbers $c_3, c_4 >0$ such that for $0< \delta \leq c_3/|\SymG|$, we obtain

\begin{equation*}
    |\PackingPointsFarFromBoundaryFD| \geq \frac{c_4 \delta^{-d}}{|\SymG|},
\end{equation*}
and this completes the proof.
\end{proof}

Now, given a strictly $\delta$-packing $\PackingPointsFarFromBoundaryFD$ of Dirichlet domain $\DirichletFD$, for each $i \in \PackingPointsFarFromBoundaryFD$, define a subset $\ClusterOrbit_i = \bigcup_{g\in \SymG} \Ball(g\cdot i,\frac{\delta}{2})$. 
It is obvious that the members of the collection $\{\ClusterOrbit_i\}_{i\in \PackingPointsFarFromBoundaryFD}$ are mutually disjoint.

Define a constant function $\ExpReward_0$, that is, $\ExpReward_0(x) = \frac{1}{3}$ for all $x\in \ArmSet$. 
For any number $\delta \in (0,\frac{2}{3})$ and for each $i\in \PackingPointsFarFromBoundaryFD$, define a function $\ExpReward_i$ as follows:

\begin{equation}
    \ExpReward_i(x) = 
    \begin{cases}
        \frac{1}{3} + \frac{\delta}{2} - \Distance(i,x),  & \text{if $x \in \Ball(g\cdot i,\frac{\delta}{2}), g \in \SymG$; }\\
        \frac{1}{3}, & \text{otherwise.}  
    \end{cases}
\end{equation}

We now show that a bandit problem instance with the expected reward function $\ExpReward_i$ is indeed an instance of ILB problems.

\begin{lemma} \label{lem: Invariant Lipschitz instances}
For each $i \in \PackingPointsFarFromBoundaryFD$, the function $\ExpReward_i$ satisfies the following:
\begin{enumerate}
    \item[(i)] $\sup_{x \in \ClusterOrbit_i} \ExpReward_i(x) - \sup_{x \in \ArmSet} \ExpReward_0(x) = \frac{\delta}{2}$, and $\ExpReward_i(x) -  \ExpReward_0(x) \leq \frac{\delta}{2}$ if $x \in \ClusterOrbit_i$.
    \item[(ii)] $\ExpReward_i$ is invariant with respect to the action of the group $\SymG$.
    \item[(iii)] $\ExpReward_i$ is Lipschitz continuous.
\end{enumerate}
\end{lemma}

\begin{proof}
Part (i) is immediately followed by the construction of $\ExpReward_i$ (\ref{eq: def of reward function in LB}).

Next, we prove part (ii). 
Given a point $i\in \PackingPointsFarFromBoundaryFD$, and arbitrary point $x\in \ArmSet$, there are two scenarios. 
First, if $x \notin g\cdot \Ball(i,\frac{\delta}{2})$ for all $ g\in \SymG$, then the expected reward function on the orbit of $x$ is equal to $\frac{1}{3}$. 
Second, if $x \in g\cdot \Ball(i,\frac{\delta}{2})$ for $g \in \SymG$, then for any $g' \in \SymG$, we have $g' \cdot x \in \Ball(g '\cdot g \cdot i, \frac{\delta}{2})$. Therefore,

\[\ExpReward_i(g'\cdot x) =  \frac{1}{3} + \frac{\delta}{2} - \Distance(g' \cdot g \cdot i, g' \cdot x) =  \frac{1}{3} + \frac{\delta}{2} - \Distance( g \cdot i,  x) = \ExpReward_i(x).\]

To prove part (iii), we need to show that $|\ExpReward_i(x) - \ExpReward_i(y)| \leq \Distance(x,y)$ for arbitrary $x,y \in \ArmSet$. 
For any arbitrary two points $x,y \in \ArmSet$, we consider the following scenarios. 

(1) if both $x,y$ in some balls $\Ball(g \cdot i, \frac{\delta}{2}), \Ball(g' \cdot i, \frac{\delta}{2})$. 
Now, if $g = g'$, that is, $x,y$ is in the same ball $\Ball(g \cdot i, \frac{\delta}{2})$, then it is obvious that

\[|\ExpReward_i(x)-\ExpReward_i(y)| =  |\Distance(g\cdot i,y)-\Distance(g\cdot i,x)| \leq \Distance(x,y). \]

Consider the case where $g \neq g'$, given the fact that $\ArmSet$ is convex, from any two distinct points $x,y$, we can draw a line segment $[x,y]$. 
Since $\Distance(x,g\cdot i) < \frac{\delta}{2}$, $\Distance(y,g\cdot i) > \frac{\delta}{2}$, and given the fact that $\Distance$ is Euclidean metric and hence continuous, there must exist a point $x' \in (x,y)$ such that $\Distance(x',g\cdot i) = \frac{\delta}{2}$ and therefore $\ExpReward_i(x') = \frac{1}{3}$. 
Similarly, there must exist $y' \in (x,y)$ such that $\Distance(y',g'\cdot i) = \frac{\delta}{2}$ and $\ExpReward_i(y') = \frac{1}{3}$. We have
\begin{equation*}
    \begin{aligned}
        |\ExpReward_i(x) - \ExpReward_i(y)| &= |\ExpReward_i(x) - \frac{1}{3} + \frac{1}{3} - \ExpReward_i(y)| \\
        & = |\ExpReward_i(x) - \ExpReward_i(x') + \ExpReward_i(y') - \ExpReward_i(y)| \\
        & \leq |\ExpReward_i(x) - \ExpReward_i(x')| + |\ExpReward_i(y') - \ExpReward_i(y)| \\
        & \leq \Distance(x,x') + \Distance(y,y') \\
        & \leq \Distance(x,y).
    \end{aligned}
\end{equation*}
Where the last inequality holds as $x',y' \in [x,y]$. 

(2) Consider the case where neither $x$ nor $y$ is in $\Ball(g'\cdot i,\frac{\delta}{2})$ for $g' \in \SymG$. 
Without loss of generality, let $y \notin \Ball(g'\cdot i,\frac{\delta}{2})$ for all $g' \in \SymG$ so that $\ExpReward_i(y) = \frac{1}{3}$, and $x \in \Ball(g\cdot i, \frac{\delta}{2})$. With the same argument as in scenario (1), there must exist $x' \in [x,y]$ such that $\Distance(x', g\cdot i) = \frac{\delta}{2}$ and $\ExpReward_i(x') = \frac{1}{3}$. Therefore,

\begin{equation*}
    \begin{aligned}
        |\ExpReward_i(x) - \ExpReward_i(y)| &= |\ExpReward_i(x) - \frac{1}{3} + \frac{1}{3} - \ExpReward_i(y)| \\
        & \leq |\ExpReward_i(x) - \ExpReward_i(x')| \\
        & \leq \Distance(x,x') \\
        & \leq \Distance(x,y). \\
    \end{aligned}
\end{equation*}

(3) If both $x,y$ do not belong to any $\Ball(g\cdot i,\frac{\delta}{2})$ for $g \in \SymG$, it is trivially true that $|\ExpReward_i(x) - \ExpReward_i(y)| = 0 \leq \Distance(x,y)$.

\end{proof}

We adopt the notation style of~\cite{Kleinberg2019_MetricBandit_Zooming}. 
Denote the collection $\ClassExpReward$ of all mean reward functions  $\ArmSet \rightarrow [0,1]$ that is \textit{Lipschitz continuous} and \textit{invariant} with respect to the action of the group $\SymG$. 
We denote $\ClassExpReward$ as class of all feasible mean reward functions. 
Consider the rewards $\{0,1\}$, that is, for each for each $\ExpReward \in \ClassExpReward$, the probability that the reward corresponding arm $x$ is $1$ is $\ExpReward(x)$.

\begin{definition}[\cite{Kleinberg2019_MetricBandit_Zooming}]
Given the set of arms $\ArmSet$ and class of all mean reward function $\ClassExpReward$. An \textit{$(k,\varepsilon)$-ensemble} is a collection of subsets $\ClassExpReward_1,...,\ClassExpReward_k \subset \ClassExpReward$ such that there exist mutually disjoint subsets $C_1,...,C_k \subset \ArmSet$ and a function $\mu_0: \ArmSet \rightarrow [\frac{1}{3}, \frac{2}{3}]$ such that for each $i= 1,...,k$ and each function $\ExpReward_i \in \ClassExpReward_i$ the following holds:
\begin{enumerate}
    \item[(i)] $\mu_i \equiv \mu_0$ on each $C_l$, $l\neq i$;
    \item[(ii)] $\sup_{ x \in C_i}\mu_i(x) - \sup_{y \in \ArmSet} \ExpReward_0(y) \geq \varepsilon$; 
    \item[(iii)] $\ExpReward_i(x) - \ExpReward_0(x) < 2\varepsilon$ if $x\in C_i$.
\end{enumerate}
\end{definition}

Now, given a strictly $\delta$-packing of the Dirichlet domain $\PackingPointsFarFromBoundaryFD$, by Lemma \ref{lem: Invariant Lipschitz instances}, it is clear that $\{\ExpReward_i\}_{i \in \PackingPointsFarFromBoundaryFD}$ is a $(|\PackingPointsFarFromBoundaryFD|,\frac{\delta}{2})$-ensemble.

\begin{lemma}[\cite{Kleinberg2019_MetricBandit_Zooming}'s Theorem 5.7] \label{lem: BAI for ensemble}
Consider invariant Lipschitz bandit problems with $\{0,1\}$ rewards. 
Let $\ClassExpReward_1,...,\ClassExpReward_k \subset \ClassExpReward$ be an $(k,\varepsilon)$-ensemble, where $k>2$ and $\varepsilon \in (0,\frac{1}{12})$. 
Then for any $\TotalRound \leq \frac{1}{128} k \varepsilon^{-2}$ and any bandit algorithm there exist at least $\frac{k}{2}$ distinct i's such that the regret of this algorithm on any mean reward function from $\ClassExpReward_i$ is at least $\frac{1}{60} \varepsilon \TotalRound$.
\end{lemma}

Now, Lemma \ref{lem: BAI for ensemble} together with Lemma \ref{lem: LB for Packing Points Far FromBoundaryFD} is sufficient to prove the following lower bound.

\begin{proof}[Proof of Theorem \ref{Thm: Minimax Lower Bound Invariant Lipschitz}]

By Lemma \ref{lem: LB for Packing Points Far FromBoundaryFD}, there exist some constants $c_1$ such that for 

\begin{equation}\label{eq: Minimax Lower Bound temp 1}
    \delta \leq \frac{c_1}{|\SymG|},
\end{equation}

we have

\begin{equation}\label{eq: Minimax Lower Bound temp 2}
    |\PackingPointsFarFromBoundaryFD| = \Omega\left(\frac{\delta^{-d}}{|\SymG|}\right).
\end{equation}

Consider an arbitrary bandit strategy $\Strategy$. 
By Theorem \ref{Thm: Minimax Lower Bound Invariant Lipschitz}, suppose $|\PackingPointsFarFromBoundaryFD| \geq 2$ and given the fact that $\{\ExpReward_i\}_{i\in \PackingPointsFarFromBoundaryFD}$ is $(|\PackingPointsFarFromBoundaryFD|,\frac{\delta}{2})$-ensembles, it follows that for any time horizon $\TotalRound$ satisfying

\begin{equation}\label{eq: Minimax Lower Bound temp 3}
  \TotalRound = \Omega(|\PackingPointsFarFromBoundaryFD|\delta^{-2}),   
\end{equation}

there must exist an index $i\in \PackingPointsFarFromBoundaryFD$ such that for expected reward function $\ExpReward_i$, the regret of this algorithm satisfies
\begin{equation}\label{eq: Minimax Lower Bound temp 4}
    \Regret_\TotalRound = \Omega\left( \delta \TotalRound \right).
\end{equation}

Substitute (\ref{eq: Minimax Lower Bound temp 1}), (\ref{eq: Minimax Lower Bound temp 2}) into (\ref{eq: Minimax Lower Bound temp 3}) we have

\begin{equation}\label{eq: Minimax Lower Bound temp 5}
    \TotalRound \geq \Omega \left(\frac{\delta^{-(d+2)}}{|\SymG|} \right) 
    \geq \Omega \left(|\SymG|^{d+1} \right).
\end{equation}

Now, suppose that $\TotalRound$ satisfies (\ref{eq: Minimax Lower Bound temp 5}), choose $\delta$ as follows

\begin{equation}\label{eq: Minimax Lower Bound temp 6}
    \delta = \left(|\SymG| \TotalRound\right)^{-\frac{1}{d+2}}.
\end{equation}

Substitute (\ref{eq: Minimax Lower Bound temp 6}) into (\ref{eq: Minimax Lower Bound temp 4}), we obtain that for $\TotalRound \geq \Omega \left(|\SymG|^{d+1} \right)$, regret of the algorithm satisfies

\begin{equation}
\begin{aligned}
    \Regret_\TotalRound &\geq \Omega\left( \delta \TotalRound \right)  \\
    &\geq \Omega\left( \left(\frac{1}{|\SymG|}\right)^{\frac{1}{d+2}} \TotalRound^{\frac{d+1}{d+2}} \right).
\end{aligned}
\end{equation}

\end{proof}

\section{Proof of Proposition \ref{prop: Proper FD of ArmSet}}
\label{Appendix: Proof of prop FD of ArmSet}

\begin{proof}[Proof of Proposition \ref{prop: Proper FD of ArmSet}]
(i) The proof of part (i) is as follows.
Since $\FundDomain$ is open in $\Euclidean^d$, $\FundDomain \cap \ArmSet$ is open in $\ArmSet$ with respect to the subspace topology if $\ArmSet$.

(ii) The proof of part (ii) is as follows.
Because $\{g\cdot \FundDomain \}_{g\in \SymG}$ are mutually disjoint, so is $\{(g\cdot \FundDomain) \cap \ArmSet\}_{g\in \SymG}$.

(iii) Part (iii) follows immediately from the fact that both $\FundDomain$ and $\ArmSet$ are convex, so their intersection is convex and connected.

(iv) Now, we prove the final part. 
We have 

\begin{equation} \label{eq: Proper FD of ArmSet - temp 1}
\begin{aligned}
        \bigcup_{g\in \SymG} g \cdot \Closure{\DirichletFD} =  \bigcup_{g\in \SymG} g \cdot \Closure{\FundDomain \cap \ArmSet}.
\end{aligned}
\end{equation}

Next, we prove that $g \cdot \Closure{\FundDomain \cap \ArmSet} = \Closure{g \cdot (\FundDomain \cap \ArmSet)}$. 
Consider the following:

\begin{equation} 
\begin{aligned}
        g \cdot \Closure{\FundDomain \cap \ArmSet} &= g\cdot \left\{\bigcap S \mid S \supseteq (\FundDomain \cap \ArmSet), \text{$S$ is closed in $\Euclidean^d$} \right\} \\
        & = \left\{\bigcap g\cdot S \mid S \supseteq (\FundDomain \cap \ArmSet), \text{$S$ is closed in $\Euclidean^d$} \right\}.
\end{aligned}
\end{equation}
Since $g$ is an isometry, $\FundDomain \cap \ArmSet \subseteq S$ if and only if $g \cdot (\FundDomain \cap \ArmSet) \subseteq g\cdot S $.
Thus, we have

\begin{equation} \label{eq: closure and intersection}
\begin{aligned}
        g \cdot \Closure{\FundDomain \cap \ArmSet}  & = \left\{\bigcap g\cdot S \mid S \supseteq (\FundDomain \cap \ArmSet), \text{$S$ is closed in $\Euclidean^d$} \right\} \\
        & = \left\{\bigcap S' \mid S' \supseteq g \cdot (\FundDomain \cap \ArmSet), \text{$S'$ is closed in $\Euclidean^d$} \right\} \\
        & = \Closure{g \cdot (\FundDomain \cap \ArmSet)}.
\end{aligned}
\end{equation}

Substitute (\ref{eq: closure and intersection}) into (\ref{eq: Proper FD of ArmSet - temp 1}), we have 
\begin{equation}
\begin{aligned}
    \bigcup_{g\in \SymG} g \cdot \Closure{\DirichletFD} &= \bigcup_{g\in \SymG} \Closure{g \cdot (\FundDomain \cap \ArmSet)} \\
    & = \Closure{\bigcup_{g\in \SymG} (g \cdot (\FundDomain \cap  \ArmSet))} \\
    & = \Closure{\bigcup_{g\in \SymG} (g \cdot \FundDomain) \cap \bigcup_{g\in \SymG} (g \cdot \ArmSet)} \\
    & = \Closure{\Euclidean^d \cap \ArmSet} \\
    &= \Closure{\ArmSet} = \ArmSet.
\end{aligned}
\end{equation}

Here, the second equality holds due to the interchangeability of a finite union and closure operator.

\end{proof}

\section{Example of Symmetry Group}
\label{Appendix: Example of Symmetry Group}
In this section, we give examples of symmetry groups to highlight the importance and challenge of the lemma \ref{lem: upper bound for packing number}.
The difficulty of proving Lemma \ref{lem: upper bound for packing number} is that the Euclidean isometries may admit (uncountably) many fixed points and non-effective points.
Fortunately, using the notion of Dirichlet domain, it can be shown that all of those fixed and non-effective points lie on the boundary of the fundamental domain.

\begin{example}\label{example: Permutation group}
As an application of Lemma \ref{lem: upper bound for packing number}, consider the case of when $\ArmSet$ is the 3D unit cube $[0,1]^3$, and $\SymG$ is the permutation group that permuted the coordinates ($|\SymG| = 3!\, = 6$).
Let us temporarily denote the coordinate variables by $x,y,z$.
Notice that $\SymG$ fixes the diagonal of the cube $\{(x,y,z) \in [0,1]^3 \mid x=y=z \}$, and acts non-effectively at points $\{(x,y,z) \in [0,1]^3 \mid (x=y) \vee (x=z) \vee (y=z)\}$.
Therefore, applying the Lemma \ref{lem: upper bound for packing number}, we obtain the number of points in the $\delta$-net $\Vertices$ to cover the fundamental domain as $|\CoveringDirichletFD| \propto \frac{\delta^{-d}}{6}$, when $\delta$ is small enough.
\end{example}

\begin{example}\label{example: Coxeter group}
Apart from the group of permutation matrices, we can apply Lemma \ref{lem: upper bound for packing number} to any finite subgroup of $\SymmetryGroup{\ArmSet}$, for example, where $\ArmSet$ is a regular polytope and $\SymG$ is symmetry group of a regular polytope known as the Coxeter group. 
In particular, consider the case where $\ArmSet$ is a 3D icosahedron whose full symmetry group $\SymmetryGroup{\ArmSet}$ is icosahedral symmetry ($|\SymmetryGroup{\ArmSet}| = 532$). 
Let $\SymG = \SymmetryGroup{\ArmSet}$, then from Lemma \ref{lem: upper bound for packing number}, we obtain a reduction in terms of the covering number $|\CoveringDirichletFD| \propto \frac{\delta^{-d}}{532}$.
\end{example}

\section{Real-world applications of ILB}
\label{Appendix: Real-world application}
An important real-world application of the Lipschitz continuous bandits, which naturally inherit symmetry, can be found in the context of online matrix factorization bandits. 
The online matrix factorization bandit problem, in turn, has important machine learning applications, such as the interactive recommendation system \cite{Kawale2015_MatrixFactoriastionBandit,Wang2019_OnlineMatrixFactorisation,Wang2017_MF_RS}, and the online dictionary learning problem \cite{Lyu20_DICLearning,Mairal2010_DICLearning}. 

For instance, consider the following matrix factorization-based recommender systems. Let $R \in \mathbb{R}^{I\times J}$ be user-item rating matrix, where $I$ and $J$ are the number of users and items, respectively.
Assume that there are $K$ features such that there exits a user-feature matrix $H \in [0,1]^{I\times K}$ and an item-feature matrix $W \in [0,1]^{J\times K}$ generating the user-item matrix $R$. 
In particular,
\[R \sim \mathrm{Pr}\{\cdot \mid H W^\top \},\]
such that $R = \mathbb{E}[HW^\top]$.

Let $n$ be the number of rounds.
In each round, the recommender system chooses a pair $(H^{(t)}$, $W^{(t)})$, then the environment returns a noisy estimation error whose expectation is as follows
\begin{equation}
    f(H^{(t)}, W^{(t)}) = \left\| \mathrm{vec}\left(H W^\top \right) - \mathrm{vec}\left( H^{(t)} W^{(t)\top} \right) \right\|,
\end{equation}
where $\mathrm{vec}$ is the vectorization operator. 
This problem can be seen as a simplified version of the recommender system; where in the full version of the recommender system, the feedback is an entry of the matrix $R$ corresponding to a chosen item and a given user for each round.

For sufficiently large $L$, for any pair $(H,W), (H',W')$ we have
\begin{equation}
    |f(H,W) - f(H',W')| \leq L \left\|\mathrm{vec}([H,W]) - \mathrm{vec}([H',W']) \right\|. 
\end{equation}

A crucial observation is that for any pair of matrices $(H^{(t)}, W^{(t)})$ chosen in round $t$, the pair $(H^{(t)}\phi, W^{(t)}\phi)$ gives the same result, where $\phi \in \mathbb{R}^{K\times K}$ is an orthogonal matrix. In particular,
\begin{equation}
    H^{(t)}\phi  (W^{(t)}\phi)^\top = H^{(t)} W^{(t)\top}.
\end{equation}

Therefore, the reward is invariant with respect to the action of the orthogonal group $O(K)$ on the pair $(H,W)$. 
This means that symmetry appears naturally in the online matrix factorization problem without any assumptions.
In fact, as observed in \cite{Li2019_SymmtryMatrixFactorisation}, the appearance of symmetry in matrix factorization gives rise to infinitely many saddle points in the optimization landscape, causing a great challenge for an optimization algorithm that is oblivious to the existence of symmetry.

Let $d = IK + JK$ be the dimension of the set of arms.
Note that our regret upper bound holds for $n > c_1 |\mathcal{G}|^{2d+2}$, for some $c_1 > 0$.
Thus, we can choose a finite subgroup $\mathcal{G} \subset O(K)$, such that $|\mathcal{G}| = c_2 n^{\frac{1}{2d+2}}$ for some $c_2>0$.
Plugging the choice of $\mathcal{G}$ into our current regret's upper bound leads to
\begin{equation}
    \mathbf{R}_n = \Tilde{\mathcal{O}}\left(L^{\frac{d}{d+2}} n^{\left(\frac{d+1}{d+2} - \frac{1}{(d+2)(2d+2)} \right)} \right).
\end{equation}

Note that while the dependence on $L$ is similar to \cite{Bubeck2011_AdaptToLipschitz},
our naive approach (i.e., direct application of ILB to the problem) is already able to reduce the growth rate of $n$ by a small factor of $d$ (from $\Tilde{\mathcal{O}}(n^{\frac{d+1}{d+2}})$).
This may indicate that the growth rate of the regret can be further reduced if we also exploit the specific properties of the matrix factorization problem, not just only the Lipschitz continuity property (as in the case of naively applying our generic ILB).
For example, by further exploiting the sparsity structure of the matrix factorization problem, state-of-the-art algorithms can achieve regret bound $\Tilde{\mathcal{O}} (n^{\frac{2}{3}})$ \cite{Sen2016_MatrixFactorisation,Jain2022_MatrixFactorisation} while the known lower bound is $\Omega({n^{\frac{1}{2}}})$ \cite{Jain2022_MatrixFactorisation}. We argue that if we combine these techniques with our ILB model (i.e., by utilising symmetry) can indeed provide further improvement and thus, provide a matching regret upper bound the matrix factorization problem.

\section{Design and Implementation of Algorithm \ref{alg:UniformMesh-N}}
\label{Appendix: Suggestion of Design and Implementation of Algorithm}
Our algorithm gives a general principle, and its computational complexity depends heavily on how to represent the set of vertices $V$. 
With a carefully designed representation of $V$, we believe that the algorithm can be carried out efficiently. 
In particular, since most of the computation burden of algorithm 1 is to search $N(x,2\delta)$ for each point $z \in \mathcal{G} \cdot x$, we suggest a design of the data structure along with a search algorithm as follows.

To design a covering of $\mathcal{X}$ and facilitate searching the location of a point simultaneously, a tree of coverings structure can be used \cite{Bubeck2011_X_armed}. 
Briefly, constructing a tree of coverings is done recursively as follows. 
\begin{itemize}
    \item[$\circ$] Denote node $i$ of the tree at depth $h$ by $\mathcal{P}_{h,i}$.
    \item[$\circ$] At depth $h=0$, the node corresponds to the original set of arms, that is, $\mathcal{P}_{0,1} = \mathcal{X}$.
    \item[$\circ$] At depth $h=1$, we divide $\mathcal{X}$ into two disjoint subsets $\mathcal{P}_{1,1}$ and $\mathcal{P}_{1,2}$, and continue to split in the next level $h=2,...$. 
    \item[$\circ$] The partition process must guarantee that at depth $h$ the diameter of $\mathcal{P}_{h,i}$ must be smaller than $c\rho^h$ for some numbers $\rho \in (0,1)$ and $c>0$; that is, the diameter of the subsets shrinks geometrically with the depth. 
    \item[$\circ$] The splitting ends when $\rho^h < 2\delta$, and let $H$ be the depth of the tree, we store all the nodes at depth $H$ as a representation of $V$.
\end{itemize}
The details of implementing tree of coverings are referred to \cite{Bubeck2011_X_armed} as it requires some extra notation.

Given a tree of coverings, for a point $z \in \mathcal{G} \cdot x$, to approximately find its neighbor $Z = \{i\in V \mid \Distance(i,z) < \epsilon\}$, we can search which node at depth $H'$ to which $z$ belongs, where the diameter of the nodes at depth $H'$ is at most $2\epsilon$.
Then we can take its successors as $Z$.  
The search process takes at most $\mathcal O \left(d\log\left(\delta^{-1}\right)\right)$.

Therefore, in each round, only $\mathcal O \left(|\mathcal{G}| d\log\left(\delta^{-1}\right)\right)$ is required to search $\mathrm{N}(x,\epsilon)$.
Although it is not exact $\mathrm{N}(i,\epsilon)$, it can be surpassed by adding an additional factor of $\epsilon$ to the confidence bound to ensure that the estimate is optimistic. 
This factor will appear in regret as a multiplicative term. 
We believe that with a more carefully designed data structure, we can find the exact $\mathrm{N}(x,\epsilon)$ efficiently.
Note that this tree structure was already used in the Lipschitz bandit literature \cite{Bubeck2011_X_armed}. 
Therefore, we conjecture that the combination of tree-based algorithm with the side observation scheme may solve both statistical and computational complexity problems simultaneously. 

In conclusion, in contrast to sampling within the fundamental domain, which is computationally intractable, our algorithm is computationally efficient by carefully designing the data structure to facilitate searching for the neighbor of orbits.

\end{document}